\documentclass[11pt]{article}
\usepackage{graphicx} % Required for inserting images
\usepackage{fullpage,amsmath,amsfonts,amssymb,amsthm}
\usepackage{xcolor}
\usepackage{natbib}
\usepackage{mdframed}
\usepackage[linesnumbered,ruled,vlined]{algorithm2e}
\SetKwInput{KwInput}{Input}                % Set the Input
\SetKwInput{KwReturn}{Return}

\SetCommentSty{mycommfont}
\usepackage{graphbox}
\usepackage{authblk}
\usepackage{hyperref}
\usepackage{booktabs}
\usepackage{multirow}
\usepackage{comment}

\usepackage{tikz}
\usepackage{pgfplots}
\pgfplotsset{compat=newest}
\newcommand{\gmmsigma}{0.1}   % σ
\newcommand{\gmmx}{0.1}       % observed x
\newcommand{\gmmt}{3}         % time-step t
\newcommand{\gmmT}{10}        % total steps T
\newcommand{\KL}{{\mathrm{KL}}}

\def\ddefloop#1{\ifx\ddefloop#1\else\ddef{#1}\expandafter\ddefloop\fi}

\def\ddef#1{\expandafter\def\csname r#1\endcsname{{\ensuremath{\mathbb{#1}}}}}
\ddefloop ABCDEFGHIJKLMNOPQRSTUVWXYZ\ddefloop

\def\ddef#1{\expandafter\def\csname c#1\endcsname{{\ensuremath{\mathcal{#1}}}}}
\ddefloop ABCDEFGHIJKLMNOPQRSTUVWXYZ\ddefloop

\newcommand{\E}{{\mathbb{E}}}
\DeclareMathOperator*{\argmin}{argmin}

\newtheorem{definition}{Definition}
\newtheorem{lemma}{Lemma}

\title{Reverse Markov Learning: \\Multi-Step Generative Models for Complex Distributions}
\author[$\star$]{Xinwei Shen}
\author[$\dag$]{Nicolai Meinshausen}
\author[$\ddag$]{Tong Zhang}
\affil[$\star$]{Department of Statistics, University of Washington}
\affil[$\dag$]{Citadel Securities}
\affil[$\ddag$]{Department of Computer Science, University of Illinois Urbana-Champaign}
\date{}

\newcommand\xinwei[1]{{\color{cyan}Xinwei: #1}}

\newtheorem{theorem}{Theorem}
\newtheorem{assumption}{Assumption}

\begin{document}

\maketitle

\begin{abstract}
Learning complex distributions is a fundamental challenge in contemporary applications. \citet{shen2024engression} introduced engression, a generative approach based on scoring rules that maps noise (and covariates, if available) directly to data. While effective, engression can struggle with highly complex distributions, such as those encountered in image data. In this work, we propose reverse Markov learning (RML), a framework that defines a general forward process transitioning from the target distribution to a known distribution (e.g., Gaussian) and then learns a reverse Markov process using multiple engression models. This reverse process reconstructs the target distribution step by step. This framework accommodates general forward processes, allows for dimension reduction, and naturally discretizes the generative process. In the special case of diffusion-based forward processes, RML provides an efficient discretization strategy for both training and inference in diffusion models. We further introduce an alternating sampling scheme to enhance post-training performance. Our statistical analysis establishes error bounds for RML and elucidates its advantages in estimation efficiency and flexibility in forward process design. Empirical results on simulated and climate data corroborate the theoretical findings, demonstrating the effectiveness of RML in capturing complex distributions. 
	
\end{abstract}

\section{Introduction}

Modern applications of statistics and machine learning increasingly require learning entire distributions, either for more comprehensive uncertainty quantification or for generating realistic data. Consider a random vector $X\in\mathbb{R}^d$. Given an i.i.d.\ sample of $X$, classical parametric distribution families or nonparametric density estimation approaches, such as the kernel density estimation, have been well studied for simple or low-dimensional distributions of $X$. Generative models such as diffusion models~\citep{sohl2015deep,song2019generative,ho2020denoising}, on the other hand, have achieved remarkable success in generating samples from complex and high-dimensional data distributions in modern problems---from image synthesis to climate modeling. 

When covariates $Y \in \mathbb{R}^p$ are available, the goal shifts to learning the conditional distribution of $X$ given $Y$, denoted by $p^*_{X|Y}$. To this end, distributional regression approaches have been developed through estimating the cumulative distribution function~\citep{foresi1995conditional,hothorn2014conditional}, density function~\citep{dunson2007bayesian}, quantile function~\citep{meinshausen2006quantile}, etc. Alternatively, conditional generative models are potentially more powerful for high-dimensional responses.

\citet{shen2024engression} introduced \emph{engression}, a generative model-based method for learning distributions in the regression context. A generative model $g(Y,\varepsilon)$ maps the covariates $Y$ (if any) and Gaussian noise $\varepsilon$ to the data space $\mathbb{R}^d$. Engression is trained by minimizing the energy loss
\begin{equation}\label{eq:engression}
    g^*\in\argmin_{g\in\mathcal{G}} \E\left[\|X-g(Y,\varepsilon)\|_2 - \frac12\|g(Y,\varepsilon) - g(Y,\varepsilon')\|_2\right],
\end{equation}
where $\varepsilon$ and $\varepsilon'$ are two i.i.d.\ draws from the standard Gaussian and $\mathcal{G}$ is a function class. It holds that $g^*(y,\varepsilon)\sim p^*_{X|Y=y}$ for any fixed $y$, when such a $g^*$ exists. See Section~\ref{sec:eng} for details. 
Engression is very simple, tractable, and has been shown to be effective across diverse real applications, but its performance can degrade when the target distribution is highly complex.

This paper aims to address the challenge of learning complex distributions by proposing a framework termed Reverse Markov Learning (RML). Inspired by the success of diffusion models, we consider a forward process from the target distribution to a known distribution such as the standard Gaussian and then learn a {reverse Markov process} via multiple engression models to match all the reverse conditional distributions with the forward process. This ensures that the final state of this reverse process matches the target distribution. %We will formalize the framework and justify its correctness in Section~\ref{sec:method}. 
Intuitively, the task of learning a complex distribution is split into learning multiple simpler (conditional) distributions, which potentially leads to statistical gains, especially for complex, less smooth distributions. Moreover, we introduce a post-training technique that alternates between the forward and backward processes to further enhance the performance. 
%The framework as well as its theoretical properties and guarantees will be formalized and justified throughout Section \ref{sec:method}-\ref{sec:analysis_specific}.

Notably, RML can be viewed as a flexible, computationally efficient generalization of diffusion models or flow matching \citep{lipman2023flow}. Indeed, we show that the continuous limit of a special case of our approach recovers flow matching. Due to the flexibility of engression in learning conditional distributions, RML allows the forward process to be defined arbitrarily. When the forward process is chosen as a diffusion process or a linear interpolation between data and noise, our framework provides a computationally more efficient technique that discretizes training and sampling of diffusion models while not suffering from any discretization error.

This generality also addresses two major computational bottlenecks of diffusion models:
\begin{itemize}\vspace{-3pt}
	\itemsep 2pt
	\item High dimensionality: Diffusion processes typically operate in the original data space, making training and sampling expensive for high-dimensional data. A common workaround is to apply diffusion in a learned latent space~\citep{rombach2022high}, but this involves lossy compression and limits the degree of dimension reduction. In contrast, RML can reduce dimensionality within the forward process, achieving substantial computational savings in both training and sampling, while exactly preserving the data distribution by reverse Markov reconstruction. 
	\item Long sampling chain: Sampling in diffusion models often requires many steps to control discretization error, though distillation techniques such as consistency models~\citep{song2023consistency} can mitigate this issue after training. RML, by construction, uses a finite-step forward process, enabling a generation algorithm with a small, fixed number of steps without sacrificing correctness.
\end{itemize}
%In addition to its generality, our approach can alleviate the two main computational burdens of diffusion models, which are of great practical concern. 
%First, the processes in diffusion models maintain the data dimension, which is usually very high in most modern applications. Thus, training and sampling are both conducted in a very high-dimensional space, which can be computationally expensive. One existing approach is to reduce the dimension in a separate step through an autoencoder and then apply diffusion models to the latent space~\citep{rombach2022high}; this usually involves lossy data compression and hence is not guaranteed to preserve the exact data distribution. In comparison, our forward process can have varying, especially reducing dimensions, which can significantly save the computational cost in both training and generation, while preserving the data distribution. Furthermore, generation in diffusion model often takes many steps, which is time consuming (although various techniques such as consistency models by \citet{song2023consistency} can be employed to alleviate this issue after training). Our framework, in contrast, is naturally discretized with a finite-step forward process so that the final generation algorithm is conducted directly in a finite number of steps. Both properties make our method computationally more appealing than diffusion models. 

The remainder of the paper is organized as follows. Section~\ref{sec:method} introduces the RML framework  its theoretical correctness, and a refinement technique. Section~\ref{sec:theory} connects RML to flow matching. Section~\ref{sec:analysis_general} develops general error bounds, and Section~\ref{sec:analysis_specific} examines more properties of RML for specific examples. Section~\ref{sec:precip} presents an application to regional precipitation prediction. Section~\ref{sec:conclusion} concludes. Proofs, experimental details, and additional results are provided in the Appendix.
%The reverse Markov learning framework is depicted in Section~\ref{sec:method}, as well as theoretical justifications on its correctness and some numerical examples. Section~\ref{sec:theory} connects our framework to flow matching. Statistical analysis for a general error bound is established in Section~\ref{sec:analysis_general} along with more detailed investigations for specific examples in Section~\ref{sec:analysis_specific}. An application to regional precipitation prediction is presented in Section~\ref{sec:precip}. Section~\ref{sec:conclusion} concludes. All the proofs, experimental details, and additional results are deferred to the Appendix.

%\input{note.tex}

\section{Preliminaries and Motivations}\label{sec:eng}
\subsection{Scoring Rule-Based Generative Models}
To quantify how well a distributional model (e.g., a generative model) fits the observed data, \citet{shen2024engression} consider the energy score~\citep{gneiting2007strictly}, a widely used proper scoring rule. For a candidate distribution $p$ and an observation $x$, the energy score is defined as
\begin{equation*}
	S(p,x) = \frac12\E\|X-X'\|_2^\beta - \E\|X-x\|_2^\beta,
\end{equation*}
where $X$ and $X'$ are two i.i.d.\ draws from $p$, and $\beta\in(0,2)$ is a hyperparameter. The energy score is a strictly proper scoring rule, that is, given $p^*$ and for any $p$, we have
\begin{equation*}
	\E_{p^*}[S(p^*,X)] \ge \E_{p^*}[S(p,X)],
\end{equation*}
where the equality holds if and only if $p=p^*$. There are many choices of strictly proper scoring rules, such as the log score and kernel score; in fact, the energy score is a special case of the kernel score. We adopt the energy score with $\beta=1$ due to its simplicity and favorable computational and theoretical properties as investigated previously by \citet{shen2024engression}.

Each proper scoring rule is associated with a corresponding distance function. In particular, the respective distance of the energy score for two distributions $p$ and $q$ is the energy distance~\citep{szekely2003statistics} defined as
\begin{equation}\label{eq:energy_distance}
    2\E\|X-\tilde{X}\|_2-\E\|X-X'\|_2-\E\|\tilde{X}-\tilde{X}'\|_2,
\end{equation}
where $X$ and $X'$ are i.i.d.\ drawn from $q$ and $\tilde{X}$ and $\tilde{X}'$ from $q$. The energy distance can also be regarded as a variant of the maximum mean discrepancy (MMD) distance~\citep{gretton2012kernel}.

Recall that in distributional regression or conditional generation, the target is the conditional distribution of $X$ given covariates $Y=y$. For a generative model $g(y,\varepsilon)$, denote by $p_g(x|y)$ its induced distribution. Applying the expected energy score to the conditional distribution gives
\begin{equation*}
	\E_{p^*(x|y)}[S(p^*(x|y), X)] \ge \E_{p^*(x|y)}[S(p_g(x|y), X)],
\end{equation*}
where the equality holds if and only if $p_g(x|y)\equiv p^*(x|y)$.
Negating the right-hand side and taking expectation over $Y$ yields the objective of the engression method in \eqref{eq:engression}. It is also implied from this inequality that when optimized, we have $g^*(y,\varepsilon)\sim p^*_{X|Y=y}$ for any $y$ in the training support; see \citet[Proposition 1]{shen2024engression} for a formal statement.

As a generative modeling approach, engression learns a single map from noise to data, similar in goal to variational autoencoders~\citep{kingma2013auto} and generative adversarial networks~\citep{goodfellow2014}. Unlike these methods, however, engression avoids variational approximations or adversarial training, both of which require a second network (encoder or discriminator). Instead, it directly solves a single minimization problem, often yielding faster convergence and robust performance, and requiring less hyperparameter tuning~\citep{shen2024engression}.

\subsection{From One-Step to Continuous-Time Generative Models}
One-step generative models offer efficient test-time computation and allow for dimensionality reduction, enabling a more structured latent space. However, many real-world data distributions, especially in modern applications, are highly complex. For such cases, a single map from noise to data may exceed the model's capacity or be difficult to learn due to statistical or optimization challenges.

As an illustrative example, consider a poorly mixed Gaussian mixture (Figure~\ref{fig:mog}, left). Applying (unconditional) engression in one step (second panel) captures the modes but yields an overly smooth generator, producing many unrealistic samples in low-density regions.

\begin{figure}
\centering
\begin{tabular}{@{}cccc@{}}
	\includegraphics[width=0.22\textwidth]{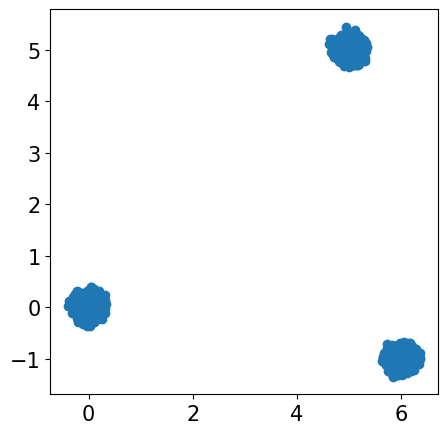} &
	\includegraphics[width=0.22\textwidth]{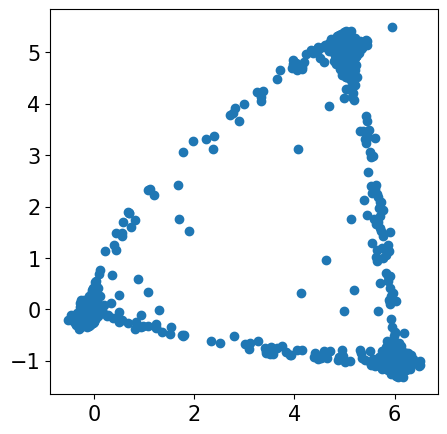} &
	\includegraphics[width=0.22\textwidth]{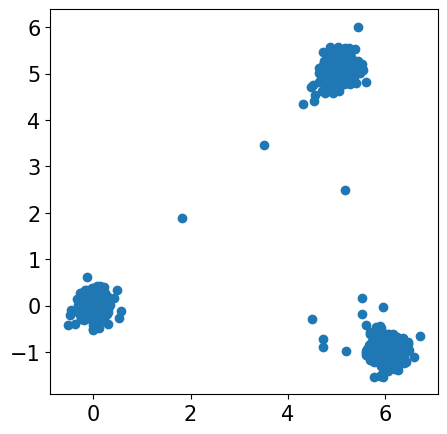} &
	\includegraphics[width=0.22\textwidth]{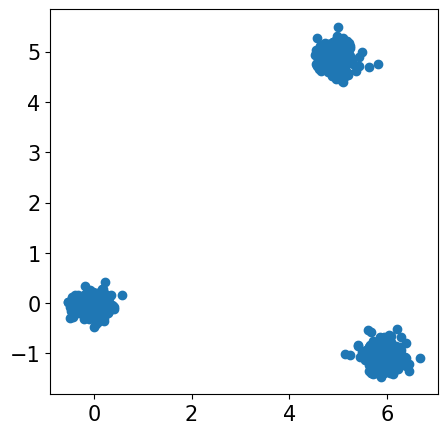} \\
	True data & Engression samples & \small{RML samples ($T=5$)} & \small{RML samples ($T=10$)}
\end{tabular}
\caption{An illustrative example of a mixture of three Gaussians, i.e., $\frac13\mathcal{N}(\mu_1,\sigma I_2)+\frac13\mathcal{N}(\mu_2,\sigma I_2)+\frac13\mathcal{N}(\mu_3,\sigma I_2)$, where $\mu_1=(0,0)$, $\mu_2=(5,5)$, $\mu_3=(6,-1)$, and $\sigma=0.1$. The plots show samples from the true distribution and estimated distributions by engression and the proposed Reverse Markov Learning (RML) method. }\label{fig:mog}
\end{figure}

Diffusion models and flow matching address this drawback by defining a forward process that gradually corrupts data with Gaussian noise and then approximating the reverse process by estimating the score function or a target vector field that generates the desired probability path. However, accurate sampling typically requires many steps to control discretization error, making inference computationally expensive. Distillation techniques, such as consistency models, have been explored to mitigate this cost post-training.

Motivated by these limitations, we propose multi-step (discrete-time) generative models that retain the representational power of diffusion models while ensuring correctness for any finite number of steps, thereby avoiding discretization error. We design the process so that intermediate distributions are substantially easier to learn than the final target and use scoring rule–based generative models, specifically engression, at each step for its simplicity and flexibility.

\section{Reverse Markov Learning}\label{sec:method}
We introduce the method for conditional distribution learning with covariates $Y$, while unconditional generation is a special case with $Y$ being an empty set. 

\subsection{Forward Stochastic Bridging Process}\label{sec:forward}

In the forward path, we consider a general $y$-conditioned stochastic process, referred to as a
stochastic bridging process defined below.
\begin{definition}
Given an unknown target distribution $p^*_{X|Y=y}$, which we would like to learn,
and a known distribution  $q^*_{X|Y=y}$ which is easy to sample from,
we call a stochastic process $\cP: \{X_t: t=0,1,\dots,T\}|Y=y$, a 
$y$-conditioned stochastic bridging process from $p^*$ to $q^*$, if it satisfies the following two conditions: 
\begin{itemize}
	\itemsep 0pt
    \item The marginal $X_0|Y=y$ is the unknown target distribution $p^*_{X|Y=y}$.
    \item   The marginal $X_T|Y=y$ is the known distribution
    $q^*_{X|Y=y}$.
\end{itemize}
\label{def:sbp}
\end{definition}
It is worth noting the generality of the forward process. While $q^*$ is usually taken as the standard Gaussian distribution in applications, we allow it to be an  arbitrary known distribution as long as it is easy to sample from. 
%\begin{equation*}
%	X_0|Y=y\sim p^*_{X|Y=y}\quad\text{and}\quad X_T|Y=y \sim  q^*_{X|Y=y}
%\end{equation*}
Moreover, the intermediate states $X_t$ ($t \notin \{0,T\}$) can in principle be defined in arbitrary ways; for each $t$, $X_t$ may have arbitrary 
dependency on $\{X_s: s \neq t\}$.
As we will notice later, the only condition for this forward process to be useful in facilitating learning the target distribution is that the process should be autocorrelated in some sense, so that the conditional distributions of $X_{t-1}$ given $X_t$ and $Y$ are relatively easy to learn. 

Below we list a few examples of a forward process including those that have appeared in the literature. 

\medskip
\noindent\textbf{Example I: diffusion process.} As in diffusion models, we can start with the data and consider a Markov process that gradually adds noise. For example, for $t=1,\dots,T$,
\begin{equation*}
    X_t | X_{t-1}\sim \mathcal{N}(\sqrt{1-\sigma_t}X_{t-1}, \sigma_t^2 I),
\end{equation*}
where $\sigma_t\in(0,1]$ controls the variance schedule and is an increasing function of $t$ such that $\sigma_T=1$. 

\medskip
\noindent\textbf{Example II: linear interpolation.} Conditional on $(X_0,X_T)$, we can define the forward process as a deterministic process
\begin{equation*}
    X_t = (1-t/T)X_0 + (t/T)X_T,
\end{equation*}
where $t=1,\dots,T$. This is commonly used as the forward process in flow matching. 

\medskip
\noindent\textbf{Example III: spatial pooling.} For spatial data like images or geographic maps, we can define the forward process based on scientifically meaningful operators such as average pooling
\begin{equation*}
    X_t = m(X_{t-1}),
\end{equation*}
where $m(\cdot)$ is the average pooling operator of a certain factor. For example, for a $r\times r$ spatial field $X$ and average pooling with a kernel size 2, $m$ is a deterministic map from $\mathbb{R}^{r\times r}$ to $\mathbb{R}^{(r/2)\times (r/2)}$. In Section~\ref{sec:precip}, we adopt this forward process in an application to climate prediction.

\subsection{Reverse Markov Sampling}

Consider a stochastic bridging process from $p^*$ to $q^*$, as in Definition~\ref{def:sbp}.
Given a time step $t \in \{1,\ldots,T\}$, 
we further denote by $p^*_{t}(x_{t-1},x_t|y)$ the joint distribution of $(X_{t-1},X_t)$ conditioned on $Y=y$,
and we denote by $p^*_{t}(x_{t-1}|x_t,y)$ the conditional distribution of $X_{t-1}|X_t,Y$. 
Using this notation, we can define a reverse $y$-conditioned Markov process 
$\{\tilde{X}_t, t=T,T-1,\dots,0\}$ as in Algorithm~\ref{alg:rms}.
The output of the algorithm is $\tilde{X}_0$.

\medskip

\begin{algorithm}[H]
\caption{Reverse Markov Sampling}\label{alg:rms}
\KwIn{Condition \(y\), bridging process $\cP$}
\KwOut{\(\tilde{X}_0\)}
Sample \(\tilde{X}_T \sim q^*_{X|Y=y}\)\;
\For{\(t = T, T-1, \ldots, 1\)}{
  Sample \(\tilde{X}_{t-1} \sim p^*_{t}(x_{t-1} \mid \tilde{X}_t, y)\)\;
}
\KwReturn{$\tilde{X}_0$}
\end{algorithm}
\medskip

 We have the following general result for this reverse Markov sampling process, which indicates that we can sample from the target distribution $p^*_{X|Y=y}$ using Algorithm~\ref{thm:rmp}, as long as we can sample from each Markov conditional distribution
 $p^*_{t}(x_{t-1} | x_t, y)$. 

 \begin{theorem}
 Consider $\tilde{X}_t$ generated according to Algorithm~\ref{alg:rms}.
 Let $\tilde{p}_t(\tilde{x}_{t-1},\tilde{x}_t|y)$
 be the joint density of $(\tilde{X}_{t-1},\tilde{X}_t)|Y=y$. 
 Then for all $t =1,\ldots,T$, we have
 \[
 \tilde{p}_t(x_{t-1},x_t|y)=p^*_t(x_{t-1},x_t|y) .
 \]
 This implies that the sampled distribution
 $\tilde{X}_0|Y=y$ is the same as that of the target distribution
 $p^*_{X|Y}(\cdot|y)$.
 \label{thm:rmp}
 \end{theorem}

% This result implies that we can sample from the target distribution using Algorithm~\ref{thm:rmp}, as long as we can sample from each Markov conditional distribution
% $p^*_{t}(x_{t-1} \mid \tilde{X}_t, y)$. 
 
\subsection{Training Algorithm}

In order to implement Algorithm~\ref{alg:rms} in practice, we need to sample from the distribution 
\[
p^*_{t}(x_{t-1} | x_t, y). 
\]
We refer to this sampling problem as conditional reverse Markov sampling. Since this conditional distribution is generally unknown, practical implementation requires learning  it from the data.

We propose to learn this reverse Markov process via multi-step engression. Specifically, for each $t=1,\dots,T$, we aim to find a function $g_t$ such that the distribution of $g_t(x_t,y,\varepsilon_t)$ given $X_t=x_t$ and $Y=y$ matches the true one $p^*_{t}(x_{t-1}|x_t,y)$ in the forward process. Here $\varepsilon_t$ is assumed to be drawn from a Gaussian distribution. Denote $g(t,x_t,y,\varepsilon_t)=g_t(x_t,y,\varepsilon_t)$.

To this end, we define
\begin{equation*}
    g^*\in\argmin_{g}\E\left[\|X_{t-1}-g_t(X_t,Y,\varepsilon_t)\|_2 - \frac12\|g_t(X_t,Y,\varepsilon_t) - g_t(X_t,Y,\varepsilon_t')\|_2\right],
\end{equation*}
where $\varepsilon_t$ and $\varepsilon_t'$ are two i.i.d.\ draws from the standard Gaussian, and the expectation is taken over all random variables including $t\sim\mathrm{Unif}\{1,\dots,T\}$. 

When all $X_t$'s share the same dimension, we can parametrize all $g_t$'s as a shared function of $t,x_t,y$ and $\varepsilon_t$, similar to the time embedding in a score network in diffusion models. When we allow varying dimensions, we need separate models for $g_t$'s with different dimensions. %We summarize the training procedure given a finite sample of $p^*_{X,Y}$ in Algorithm~\ref{alg:rml}.

We note that in standard training data, we observe a set of training data
\[
\mathcal{S}=\{(x^1,y^1),\ldots,(x^n,y^n)\} \sim p^*_{X,Y} .
\]
For each sample $(x^i,y^i) \in \cS$, we can then augment it to 
\[
(x_0^i,x_1^i,\ldots,x_T^i,y^i) 
\]
according to the underlying stochastic bridging process $\cP$. 
In Algorithm~\ref{alg:rml}, we assume that given $\cP$, and $t \in \{1,\ldots,T\}$, we can take sample
$(x_{t-1}^i,x^i_{t},y^i)$ from $\cP$ and the training data $\cS$. Here we know that
\[
x^i_{t-1},x^i_{t}|y^i \sim p^*_t(x^i_{t-1},x^i_{t}|y^i) .
\]
With this in mind, we can state the algorithm of Reverse Markov Learning (RML) as follows. 

{\centering
\begin{minipage}{\linewidth}
\vskip 0.1in

\SetKwFor{Iterate}{Iterate}{do}{end iterate}

\begin{algorithm}[H]
\DontPrintSemicolon
\KwInput{Training sample $\mathcal{S}$, bridging process $\cP$, batch size $m$}
%\KwOut{Generators $\{g_t: t=1,\ldots,T\}$}

\Iterate{\text{until converge}}{
Sample $t\sim\mathrm{Unif}\{1,\dots,T\}$\\
\For{$i=1\ldots,m$}{
%Sample $i \sim \mathrm{Unif}\{1,\ldots,n\}$\\
Take $y^i$ from $\cS$\\
Take a sample $(x^{i}_{t-1},x^{i}_t)\sim p^*_{t}(x^i_{t-1},x^i_{t}|y^i)$\\
Sample $\varepsilon_{i},\varepsilon'_{i}\sim\mathcal{N}(0,I)$
}
Update parameters of $g_t$ by descending the gradients of
\begin{equation*}
	\frac{1}{m}\sum_{i=1}^m\left[\left\|x^{i}_{t-1} - g_t(x^{i}_t,y^i,\varepsilon_{i})\right\|_2 - \frac12\left\|g_t(x^{i}_{t},y^i,\varepsilon_{i}) - g_t(x^{i}_{t},y^i,\varepsilon'_{i})\right\|_2\right]
\end{equation*}
}
\KwReturn{Generators $\hat{g}_t:=g_t,t=1,\dots,T$}
\caption{Reverse Markov Learning (RML)}
\label{alg:rml}
\end{algorithm}
\end{minipage}
\vskip 0.1in
\par
}

\subsection{Reverse Markov Sampling with Learned Generator}

Using the generators $\{\hat{g}_t: t=1,\ldots,T\}$ learned from Algorithm~\ref{alg:rml}, we can instantiate Algorithm~\ref{alg:rms}
in Algorithm~\ref{alg:rmg} as the practical sampling algorithm. 

\medskip

\begin{algorithm}[H]
\caption{Reverse Markov Generation}\label{alg:rmg}
\KwIn{Condition \(y\), $q^*$, generators $\{\hat{g}_t: t=1,\ldots,T\}$}
\KwOut{\(\tilde{X}_0\)}
Sample \(\tilde{X}_T \sim q^*_{X|Y=y}\)\;
\For{\(t = T, T-1, \ldots, 1\)}{
$\varepsilon_t \sim N(0,I)$\\
  $\tilde{X}_{t-1} =\hat{g}_t(\tilde{X}_t,y,\varepsilon_t)$
}
\KwReturn{$\tilde{X}_0$}
\end{algorithm}
\medskip

As an immediate demonstration, for the illustrative example in Figure~\ref{fig:mog}, we show the generated samples from our method in the two plots on the right. With 5 steps, RML already exhibits a significant advantage over engression (RML with one time step). With 10 steps, RML can produce samples from the true mixture of Gaussians very well. In the Appendix, we further show the samples from Reverse Markov Generation at intermediate steps, from which we can see how the process evolves so that each reverse conditional distribution $p^*_{t}(x_{t-1}|x_t,y)$ is easier to learn than the original target $p^*_{X|Y=y}$.

%\begin{figure}
%	\centering
%	\includegraphics[width=\textwidth]{fig/mog_gen_mine_T10_inter.png}
%	\caption{Samples at intermediate steps by Reverse Markov Learning with $T=10$ for the illustrative example in Figure~\ref{fig:mog}.}\label{fig:mog_inter_t10}
%\end{figure}

\subsection{Alternating Markov Generation}

The forward stochastic bridging process is considered to be general, where given a generated sample $\tilde{X}_0$ at time $t=0$, we can generate samples $\tilde{X}_{t'}$ at time $t=t'$ using the forward processes. Assume that for each $t \in \{0,1,\ldots,T\}$, there exists $s(t) \in \{0,\ldots,t\}$ so that 
$X_t$ can be generated from $X_{s(t)}$ using the forward process. We note that we may always take $s(t)=0$, but it can be convenient to allow $s(t)>0$. Note that if $s(t)=t$, then the generation of $X_t$ from $X_{s(t)}$ is simply the identity map. 

\begin{assumption}
Assume for any $t \in \{0,1,\ldots,T\}$, there exists $s(t) \in \{0,\ldots,t\}$ so that the forward process $\cP$ in Definition~\ref{def:sbp} implies a known process $\cP'$ that can take $y$ and a sample $X_{s(t)}'$ with $y$-conditioned marginal $p_{s(t)}(\cdot|y)$, and generate a sample
$X_t'$ with $y$-conditioned marginal $p_t(\cdot|y)$. 
\label{assump:sbp-markov}
\end{assumption}

One may leverage Assumption~\ref{assump:sbp-markov} to design a generalized reverse Markov generation process, which alternates between reverse generation and forward process.
As we will see later, this procedure may have some theoretical advantages.

\begin{algorithm}[H]
\caption{Alternating Reverse Markov Generation}\label{alg:rmg-alt}
\KwIn{Condition \(y\), $q^*$, $\cP'$, generators $\{\hat{g}_t: t=1,\ldots,T\}$}
\KwOut{\(\tilde{X}_0\)}
Sample \(\tilde{X}_T \sim q^*_{X|Y=y}\)\;
\For{\(t = T, T-1, \ldots, 1\)}{
$\tilde{X}_t'=\tilde{X}_t$ \\
\For{$s= t, \ldots, s(t-1)+1$}{
$\varepsilon_s \sim N(0,I)$\\
  $\tilde{X}_{s-1}' =\hat{g}_{s}(\tilde{X}_s',y,\varepsilon_s)$
}
Use $\cP'$ to generate $\tilde{X}_{t-1}$ from $\tilde{X}_{s(t-1)}'$
}
\KwReturn{$\tilde{X}_0$}
\end{algorithm}

Note that in Algorithm~\ref{alg:rmg-alt}, we may also generate $\tilde{X}_{s(t-1)}'$ directly from $\tilde{X}_t'=\tilde{X}_t$ by learning a generator $\hat{g}_t$ in one-shot. 
We keep this iterative generation process to be consistent with the same training method, leaving only the generation process modified. It can be easily seen that by taking $s(t)=t$, we recover Algorithm~\ref{alg:rmg} from Algorithm~\ref{alg:rmg-alt}.

Similar to the proof of Theorem~\ref{thm:rmp}, it is also easy to see that the following theorem holds true. 
\begin{theorem}
Assume that $\hat{g}_t$ precisely learns the true reverse sampling conditional probability $p_t^*(x_{t-1}|x_t,y)$ for each $t$, then $\tilde{X}_0|Y=y$ is drawn from the target distribution $p_{X|Y}^*(\cdot|Y=y)$. 
\label{thm:rmp-alt}
\end{theorem}

As we will see in Section~\ref{sec:analysis_general}, under certain models, the error in Theorem~\ref{thm:rmp-alt} can be well controlled even if $\hat{g}_t$ only learns $p_t^*(x_{t-1}|x_t,y)$ approximately. Intuitively, the forward process can correct some errors of the learned $\hat{g}$ in Algorithm~\ref{alg:rmg-alt} so that the resulting procedure is more robust. Below we first show some empirical evidence. 

Recall the illustrative example in Figure~\ref{fig:mog}, where the performance of RML with 5 time steps is not satisfying enough since it is producing more samples in the low-density region. We apply the alternating scheme to the fitted model; specifically, we apply Algorithm~\ref{alg:rmg-alt} for generation with $s(t)\equiv0$. Figure~\ref{fig:bounce} shows the generated samples at $s=1$ for $t=5,4,\dots,1$ in Algorithm~\ref{alg:rmg-alt}, respectively. We can see almost each time alternating helps improve the performance incrementally, and eventually the full alternating algorithm gets rid of all low-density samples and learns the distribution well. 

\begin{figure}
\centering
	\includegraphics[width=\textwidth]{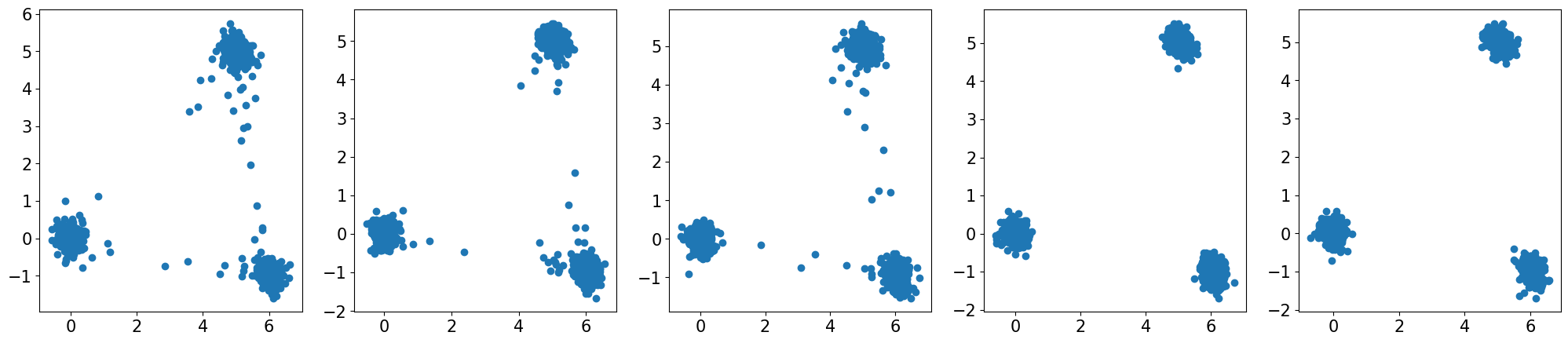}
	\caption{Generated samples at intermediate steps by Alternating Reverse Markov Generation in Algorithm~\ref{alg:rmg-alt} applied to the RML model trained with $T=5$ for the example in Figure~\ref{fig:mog}. From the left to right shows $t=5,4,\dots,1$.}\label{fig:bounce}
\end{figure}

%Additional Extension and 
\section{Relation to Flow Matching}\label{sec:theory}

We can also include additional variables, denoted by $Z$, into the condition of the general forward stochastic process as
\[
\{X_0, X_1,\ldots,X_T\}| Y=y, Z=z .
\]
Here, $Z$ represents auxiliary information available during training but not observed at inference, and thus cannot be used in the reverse generation process. Our algorithms remain directly applicable in this setting by simply ignoring $Z$ in the reverse Markov process, and Theorem~\ref{thm:rmp} continues to hold.

In what follows, we present a specific forward stochastic process whose continuous-time limit coincides with the flow matching method. This shows that our generation process can be viewed as a generalized discrete-time version of flow matching. Without loss of generality, we assume all target-dependent information is encoded in $Z$. For instance, one may take $Z = (X_0, \varepsilon)$, conditioned on $Y$, where $\varepsilon \sim q^*_{X|Y}$.

Given $Z$, the conditioning variable $y$, and a continuous time step $s \in [0,1]$, define the vector function
 \[
 h(Z,y,s) = (1-s)X_0 + s \varepsilon .
 \]
In this construction, $h(Z,y,0)$ has distribution $p^*_{X|Y}$ and $h(Z,y,1)$ has distribution $q^*_{X|Y}$. We now generalize this setup with the following definition.

\subsection{Continuous Differentiable Bridge Function}

The following definition can be used to define a continuous flow matching method.

\begin{definition}
Consider $y \in \cY$. 
Let $Z|Y=y$ be a $y$-conditioned random variable defined on $\cZ$  according to $\cP_{Z|Y}$,
and let $X$ be a $y$-conditioned random variable on $\rR^d$.
Let $p^*_{X|Y}$ and $q^*_{X|Y}$ be two distributions on $\rR^d$.
Consider a function 
\[
h(Z,y,s) : \cZ \times \cY \times [0,1] \to \rR^d 
\]
that is continuously differentiable in $s \in [0,1]$.
We call $(h,\cP_{Z|Y})$
a differentiable bridge function from $p^*$ to $q^*$ if for all $y \in \cY$:
\[
h(Z,y,0) : Z\sim \cP_{Z|Y}(\cdot|Y=y)\text{ has the same distribution as }  p^*_{X|Y}(X|Y=y) ,
\]
and
\[
h(Z,y,1): Z|Y=y \sim \cP_{Z|Y}(\cdot|Y=y) \text{ has the same distribution as }  q^*_{X|Y}(X|Y=y) .
\]
\label{def:diff_bridge}
\end{definition}

The following theorem characterizes the continuous flow matching process of
\cite{lipman2022flow} using the terminology of this paper.

\begin{theorem}\label{thm:continuous}
Let $p^*_{X|Y}$ and $q^*_{X|Y}$ be two distributions on $\rR^d$.
Let $(h,\cP_{Z|Y})$ be a differentiable bridge from $p^*$ to $q^*$. 
Define
\[
g(\tilde{x},y,s)= \E_{Z \sim \cP_{Z|Y}(\cdot|Y=y)}\left[ \frac{\partial}{\partial s} h(Z,y,s)\Big| h(Z,y,s)=\tilde{x}\right] .
\]
Let $\tilde{x}(x_1,y,s)$ be the solution of the differentiable equation (conditioned on $y$):
\begin{equation}
    \frac{\partial}{\partial s} \tilde{x}(x_1, y,s) =  g(\tilde{x}(x_1, y,s),y, s) , 
    \quad 
    \tilde{x}(x_1,y,1)=x_1.  \label{eq:flow-ode}
  \end{equation}
  Then the random variable
  \[
     h(Z,y,s): Z \sim \cP_{Z|Y}(\cdot|Y=y)
  \]
   has the same distribution as that of
  \[
    \tilde{x}(\tilde{X}_1,y,s), \qquad \tilde{X}_1 \sim q^*_{X|Y}(\cdot|Y=y) .
  \]
\end{theorem}

The theorem can be proved by comparing the Fokker-Planck equations of the two processes. We will skip the proof. In this paper, we will treat it as the continuous limiting situation of Reverse Markov Sampling, as stated in Theorem~\ref{thm:continuous-discretized} below. 
Therefore the proof of Theorem~\ref{thm:continuous-discretized} directly implies Theorem~\ref{thm:continuous} under suitable regularity conditions. 

  We note that the flow matching method relies on the following optimization formula to learn $ g(\tilde{x},y, s)$ from data
  \begin{equation}
  \hat{g}=\argmin_g  \E_{Y} 
  \E_{Z \sim \cP_{Z|Y}(\cdot|Y)}
  \E_{s \sim U[0,1]} 
  \; \left[ \left(g(h(Z,Y,s),y,s)- \frac{\partial}{\partial s} h(Z,Y,s)\right)^2\right]  . 
  \label{eq:fm-learn}
  \end{equation}

    For example, if we observe examples $(X,Y)$ in training, and let $Z=(X,\varepsilon)$ with $\varepsilon \sim N(0,I)$
    $h(Z,y,s)=(1-s) X + s \varepsilon$, then the learning algorithm tries to minimize
    \[
    \hat{g}=\argmin_g  \E_{X,Y} \E_{\varepsilon \sim N(0,I)}
  \E_{s \sim U[0,1]} 
  \; \left[ \left(g((1-s)X+s\varepsilon,y,s)-  (\varepsilon-X)\right)^2\right]  . 
    \]

Using the $\hat{g}$ learned from \eqref{eq:fm-learn}, we can use the following generation process which discretizes the flow-ODE \eqref{eq:flow-ode}.

\medskip

\begin{algorithm}[H]
\caption{Reverse Flow-ODE Generation}\label{alg:fmg}
\KwIn{Condition \(y\), $q^*$, generators $\{\hat{g}(\tilde{x},y,s)\}$}
\KwOut{\(\tilde{X}_0\)}
Sample \(\tilde{X}_T \sim q^*_{X|Y}(\cdot|Y=y)\)\;
\For{\(t = T, T-1, \ldots, 1\)}{
  $\tilde{X}_{t-1} =\tilde{X}_t- (1/T) \hat{g}(\tilde{X}_t,y,t/T)$
}
\KwReturn{$\tilde{X}_0$}
\end{algorithm}

\subsection{Discretization of Continuous Flow Matching}

It is easy to see that we can define a stochastic bridging process from $p^*$ to $q^*$ using a differentiable bridge function as follows.
\begin{definition}
Let $p^*_{X|Y}$ and $q^*_{X|Y}$ be two distributions on $\rR^d$.
Let $(h,\cP_{Z|Y})$ be a differentiable bridge function from $p^*$ to $q^*$. 

Given positive integer $T>0$ and define a $y$-conditioned stochastic process as
\[
 \left\{X_t = h(Z,y,t/T) : t=0,1,\ldots,T\right\},  \;  Z \sim \cP_{Z|Y}(\cdot|Y=y).
\]
Then it is a stochastic bridging process from $p^*$ to $q^*$ implied by $(h,\cP_{Z|Y})$. 
\label{def:implied_sbp}
\end{definition}

Before stating the convergence result, we introduce the Wasserstein distance below.  
\begin{definition}[Wasserstein distance]
Let \(\mu,\nu\) be probability measures on $\rR^d$ with finite second moments, the Wasserstein-2 distance between \(\mu\) and \(\nu\) is defined by
\[
W_2(\mu,\nu)
:=
\biggl(
\inf_{\gamma \in \Pi(\mu,\nu)}
\int_{\mathbb{R}^d\times\mathbb{R}^d}
\|x - y\|_2^2
\,\mathrm{d}\gamma(x,y)
\biggr)^{1/2},
\]
where
$\Pi(\mu,\nu)$ is the set of probability measures $\gamma$ on $\rR^d \times \rR^d$ (referred to as a coupling of $\mu$ and $\nu$) such that
$\forall A\subset \rR^d,
\gamma(A\times\rR^d)=\mu(A),\;
\gamma(\rR^d\times A)=\nu(A)$.
\end{definition}

We have the following theorem, which shows that continuous flow matching is a limiting situation of our method.

\begin{theorem}\label{thm:continuous-discretized}
Let $p^*_{X|Y}$ and $q^*_{X|Y}$ be two distributions on $\rR^d$.
Let $(h,\cP_{Z|Y})$ be a differentiable bridge from $p^*$ to $q^*$. 
Define its implied stochastic bridging process 
\[
 \{X_t = h(Z,y,t/T) : t=0,1,\ldots,T\},  \;  Z \sim \cP_{Z|Y}(\cdot|Y=y).
\]
  Assume that $\frac{\partial}{\partial s} h(Z,y,s)\Big| h(Z,y,s)=\tilde{x}$ has uniformly bounded variance, and $g$ defined in Theorem~\ref{thm:continuous} is uniformly continuous,  then as $T \to \infty$,
  the distribution of $\tilde{X}_t$ converges to that of $X_t$ in Wasserstein distance, where
    $\tilde{X}_t$ is generated according to  Algorithm~\ref{alg:fmg} with $\hat{g}=g$. 
\end{theorem}

Our results establish a connection between the proposed method and the continuous flow matching approach, as well as its flow-ODE-based generation procedure in Algorithm~\ref{alg:fmg}. An advantage of Algorithm~\ref{alg:rmg} is that it is correct for any fixed, finite $T$. Increasing $T$ makes the reverse Markov process more deterministic, as shown in the proof of Theorem~\ref{thm:continuous-discretized}, and thus makes the learning process easier. Importantly, however, the correctness of Algorithms~\ref{alg:rmg} and~\ref{alg:rml} does not require $T \to \infty$. In contrast, both Algorithm~\ref{alg:fmg} and the continuous flow matching method require the limit $T \to \infty$ for correctness; any finite $T$ would introduce discretization error.

\section{Statistical Error Analysis }%of Reverse Markov Generation}
\label{sec:analysis_general}

We have shown that if the reverse Markov distributions are learned perfectly with the generators $\hat{g}_t$ from Algorithm~\ref{alg:rml}, then both Algorithm~\ref{alg:rmg} and Algorithm~\ref{alg:rmg-alt} can generate the exact target distribution. However, in practice, imperfect learning introduces error. In this section, we analyze the propagation of such errors theoretically.
In the following, we assume that Algorithm~\ref{alg:rml} achieves small Wasserstein distance errors for any conditional probability. This is supported by the theoretical results of \cite{modeste2024characterization}.

\begin{assumption} \label{assump:wasserstein}
    We assume that the conditional probability
    $\hat{p}_t(X_{t-1}|X_t=x,y)$ induced by the generator $\hat{g}_t(x,y,\epsilon)$ learned from Algorithm~\ref{alg:rml}, with $\epsilon \sim \cN(0,I)$, satisfies the following condition for all $t$:
    \[
    \E_{y \sim p^*}  \E_{x \sim p_t^*(\cdot|y)} \big[W_2(\hat{p}_t(\cdot|x,y),p_t^*(\cdot|x,y))^2\big] \leq \delta_t^2 . 
    \]
\end{assumption}

In addition, we will make the following assumption.
\begin{assumption}[Generator Smoothness] \label{assump:rml-smooth}
Assume the learned generator $\hat{g}_t$ satisfies a Lipschitz-type  condition.
There exists $L_t$ so that 
\[ 
\|\hat{g}_t(x, y,\varepsilon)- \hat{g}_t(x', y,\varepsilon) \|_2 \leq L_t\cdot \|x'-x\|_2 \qquad
\forall x, x', y, \varepsilon .
\]
\end{assumption}

\begin{lemma}
Consider $0 \leq s<t \leq T$. Given $X_t \sim \hat{p}_t$ and $y \sim p_*$. Assume $\tilde{X}_s$ is generated by applying reverse Markov generation of $\hat{g}_t, \ldots, \hat{g}_{s+1}$ from $\tilde{X}_t$.  Let the marginal distribution of $\tilde{X}_s$ be 
$\hat{p}_s$. Let $\E_{y \sim p^*}\big[W_2(\hat{p}_s(\cdot|y), p_s^*(\cdot|y))^2\big] = u_s^2$ and $L_{s}^{t'} = \prod_{s'=s+1}^{t'} L_{s'}$. 
%\[
%L_{s}^{t'} = \prod_{s'=s+1}^{t'} L_{s'} .
%\]
Then
\[
u_s \leq \sum_{t'=s+1}^t  L_{s}^{t'-1} \delta_{t'} + L_{s'}^t u_t .
\]
\label{lem:func-compose}
\end{lemma}

The following result is a direct consequence of the lemma.
\begin{theorem}
 Algorithm~\ref{alg:rmg} leads to a sample $\tilde{X}_0$ from $\hat{p}_0$ that satisfies
 \[
 \left(\E_{y \sim p^*} \big[W_2(\hat{p}_0(\cdot|y), p_0^*(\cdot|y))^2\big]\right)^{1/2}
 \leq \sum_{t=0}^{T-1} L_0^t \delta_{t+1} .
 \]
 \label{thm:rmg-error}
\end{theorem}
This implies that Algorithm~\ref{alg:rmg} leads to a sample $\tilde{X}_0$ with small Wasserstein error with respect to the target distribution if $L_{0}^t$ is small for all $t$. However, due to the product form $L_0^t$ can be potentially large. 
Algorithm~\ref{alg:rmg-alt} can be used to address this issue. In particular, we have
\begin{theorem}\label{thm:alternating}
Assume that $s(t)<t$ for $t \geq 1$, and the induced forward process $\cP'$ in Algorithm~\ref{alg:rmg-alt} from $s(t)$ to $t$ can be achieved using diffusion:
\[
X_t = f_t(X_{s(t)},\epsilon')  + \beta_t \epsilon, \qquad \epsilon \sim N(0,I) , 
\]
where $f_t(x,\epsilon')$ is known $\alpha_t$-Lipschitz function with respect to $x$, and $\epsilon'$ is an arbitrary random variable. $\beta_t>0$ is known. 
Then we have
\[
\E_{y \sim p^*} \KL(p_1^*(\cdot|y)||\hat{p}_1(\cdot|y)) \leq \sum_{t=1}^{T-1} \frac{\alpha_t^2 \left(\sum_{t'=s(t)+1}^{t} L_{s(t)}^{t'} \delta_{t'+1}\right)^2}{2 \beta_t^2}  ,
\]
where $p_1^*$ is the distribution of $X_1|y$ and $\hat{p}_1$ is the distribution of $\tilde{X}_1|y$ from Algorithm~\ref{alg:rmg-alt}. 
\end{theorem}

In particular, if we take $s(t)=t-1$, then 
\[
\E_{y \sim p_*} \KL(p_1^*(\cdot|y), \hat{p}_1(\cdot|y)) \leq \sum_{t=1}^{T-1} \frac{\alpha_t^2 \left(\delta_{t}+ L_t \delta_{t+1}\right)^2}{2 \beta_t^2}  .
\]
This implies that by the alternating scheme, we can avoid multiplicative of Lipschitz constants $L_t$ as in Theorem~\ref{thm:rmg-error}. 
We also note that the result can bound the KL-divergence between $\hat{p}_1$ and $p_1^*$, but not $\hat{p}_0$ and $p_0^*$. This is because Smoothing with nonzero $\beta_t$ allows KL-divergence to be well-defined. Without adding noise, 
the KL divergence $\hat{p}_0$ and $p_0^*$ may not be well defined. Nevertheless, we can use a technique similar to 
Theorem~\ref{thm:rmg-error} to bound the last step error from $t=1$ to $t=0$.

\section{Statistical Analysis for Specific Examples}
\label{sec:analysis_specific}
In this section, we analyze two specific examples. Section~\ref{sec:linsem} focuses on estimating the adjacency matrix in linear structural equation models, which are widely studied in the causal inference literature \cite[e.g.,][]{shimizu2006linear}. We compare the RML estimator with one-step engression and maximum likelihood estimation in terms of estimation efficiency. Motivated by these results, we also explore an RML-enhanced statistical distance aimed at improving the power of hypothesis tests. In Section~\ref{sec:forward_analysis}, we examine different forward processes in a Gaussian mixture model and show how their choice can affect learning through the smoothness of the induced reverse conditionals, underscoring RML’s advantage in allowing flexible forward-process design.

\subsection{Benefits of RML in Linear Structural Equations}\label{sec:linsem}
We consider a linear structural equation model 
\begin{equation}\label{eq:linscm}
	X=B^*X+\varepsilon
\end{equation}
where $B^*\in\mathbb{R}^{d\times d}$ is a strictly lower triangular matrix and $\varepsilon\sim\cN(0,I_d)$. The goal is to estimate the adjacency matrix $B^*$. We compare the statistical property of two estimators---(one-step) engression and RML---against the optimal maximum likelihood estimation. 

For engression, the one-step generative model is of the form $g(\varepsilon)=(I-B)^{-1}\varepsilon$, where $B\in\mathbb{R}^{d\times d}$ is a strictly lower triangular matrix. The engression loss, in this case, is given by $$\ell_{\mathrm{eng}}(B;x)=\E\|x - (I-B)^{-1}\varepsilon\|_2-\frac{1}{2}\E\|(I-B)^{-1}(\varepsilon-\varepsilon')\|_2.$$ Given an iid sample $X^1,\dots,X^n$ generated from model \eqref{eq:linscm}, the engression estimator for $B^*$ is defined as
\begin{equation*}
	\hat{B}_{\mathrm{eng}} = \argmin_{B}\sum_{i=1}^n \ell_{\mathrm{eng}}(B;X^i).
\end{equation*}
%The trace of the asymptotic variance of vec$(\hat{B}_{\mathrm{eng}})$ is given by\xinwei{double check}
%\begin{equation*}
%	\frac{1}{d\alpha_d^2}\left(\frac{d(d-1)}{2}+\sum_{i=2}^{d-1}\sum_{j=1}^{i-1}(d-i){B_{ij}^*}^2\right),
%\end{equation*}
%where 
%\begin{equation*}
%	\alpha_d=\frac{1}{\sqrt{\pi}}-\frac{\kappa_d}{2}\quad \text{and}\quad \kappa_d = \frac{\sqrt2\Gamma\big(\frac{d-1}{2}\big)}{\Gamma\big(\frac{d}{2}\big)}.
%\end{equation*}

For RML, we define the forward process by dropping one dimension per step, i.e.,
\begin{equation}\label{eq:sem_forward}
\begin{tabular}{ccccc}
	$t=0$ & $t=1$ & $\dots$ & $t=d-1$ & $t=d$\\
	$\begin{pmatrix}
		X_1 \\ \vdots \\ X_{d-1} \\ X_d
	\end{pmatrix}$ &
	$\begin{pmatrix}
		X_1 \\ \vdots \\ X_{d-1} 
	\end{pmatrix}$ & $\dots$ &
	$X_1$ &
	$\varepsilon$
\end{tabular}
\end{equation}
%\begin{center}
%\begin{tabular}{ccccc}
%	$t=0$ & $t=1$ & $\dots$ & $t=d-1$ & $t=d$\\
%	$\begin{pmatrix}
%		X_1 \\ \vdots \\ X_{d-1} \\ X_d
%	\end{pmatrix}$ &
%	$\begin{pmatrix}
%		X_1 \\ \vdots \\ X_{d-1} 
%	\end{pmatrix}$ & $\dots$ &
%	$X_1$ &
%	$\varepsilon$
%\end{tabular}
%\end{center}
Denote ${b^*_k}^\top=B^*_{k,1:(k-1)}$ the first $k-1$ elements of the $k$-th row of $B$. We note that
\begin{equation*}
	X_k|X_1,\dots,X_{k-1}\sim\cN({b^*_k}^\top X_{1:(k-1)}, 1).
\end{equation*}
%For each step $t$, consider generative model $g_t(x_1,\dots,x_{d-t},\varepsilon_t)=b_t^\top x_{1:(d-t)}+\varepsilon_t$, where $b_t\in\mathbb{R}^{d-t}$ is the model parameter and $\varepsilon_t\sim\cN(0,1)$. 
For each $k=2,\dots,d$, consider generative model $g_k(x_1,\dots,x_{k-1},\varepsilon_k)=b_k^\top x_{1:(k-1)}+\varepsilon_k$, where $b_k\in\mathbb{R}^{k-1}$ is the model parameter and $\varepsilon_k\sim\cN(0,1)$. 
Consider the loss function
\begin{equation*}
	\ell_k(b_k;x_1,\dots,x_k) = \E\|x_k-(b_k^\top x_{1:(k-1)}+\varepsilon_k)\|_2.
\end{equation*}
The estimator for step $d-k+1$ is defined by
\begin{equation*}
	\hat{b}_k = \argmin_{b_k} \sum_{i=1}^n\ell_t(b_k;X^i_1,\dots,X^i_k).
\end{equation*}
To pile up the estimators at all steps, note that each of them estimates a different row of $B^*$ after being augmented with 0, i.e., denoting $\mathbf{\hat{b}_k}:=(\hat{b}_k^\top,\mathbf{0}_{d-k}^\top)^\top\in\mathbb{R}^d$. Then the RML estimator of $B^*$ is defined as
\begin{equation*}
	\hat{B}_{\mathrm{RML}} = 
	\begin{pmatrix}
		\mathbf{0} & \mathbf{\hat{b}_2} & \dots & \mathbf{\hat{b}_d}
	\end{pmatrix}^\top.
\end{equation*}

\medskip
\noindent\textbf{Estimation efficiency.}\quad
The following results reflect the asymptotic efficiency of RML and engression compared to that of MLE. Note that for engression, the analytical form is too complicated, so we only compute it approximately for $d=2$. Later we provide some numerical results for $d>2$.
\begin{theorem}\label{thm:asy_var}
	All the estimators are consistent. The trace of the asymptotic variance of $\mathrm{vec}(\hat{B}_{\mathrm{RML}})$ is 
\begin{equation*}
	\frac{\pi}{3}\left(\frac{d(d-1)}{2}+\sum_{i=2}^{d-1}\sum_{j=1}^{i-1}(d-i){B_{ij}^*}^2\right).
\end{equation*}
The trace of the asymptotic variance of the MLE for $B^*$ is 
\begin{equation*}
	\frac{d(d-1)}{2}+\sum_{i=2}^{d-1}\sum_{j=1}^{i-1}(d-i){B_{ij}^*}^2.
\end{equation*}
The asymptotic variance of engression for $d=2$ is approximately equal to $1.2386$.
\end{theorem}
%The trace of the asymptotic variance of $\hat{b}_k$ is 
%\begin{equation*}
%	V_k:=\frac{\pi}{3}\left(k-1+\sum_{i=2}^{k-1}\sum_{j=1}^{i-1}{B^*_{ij}}^2\right). 
%\end{equation*}

There are two main implications. First, in this case, the RML estimator is only slightly (by a factor of $\pi/3$) less efficient than MLE for any $d$, suggesting that RML, despite being a more flexible generative approach, retains most of the statistical efficiency.  Second, when $d=2$, the RML estimator has an asymptotic variance of $\pi/3\approx1.0472$ which is lower than that of engression.  

For higher dimensions, we run simulations using gradient descent with random initialization for both engression and RML, and the closed-form solution for MLE. We randomly generate a true adjacency matrix $B^*$ and, for each chosen dimension and sample size, repeatedly simulate data and estimate the parameters over 100 replications. We evaluate performance using the sum of squared biases and the sum of variances over all $d(d-1)/2$ parameters.

Table~\ref{tab:sem} reports the biases and variances, and Figure~\ref{fig:sem_ratio} plots the ratio of total variance for RML or engression relative to MLE, as a measure of relative efficiency. All methods are consistent, but RML nearly matches the efficiency of MLE---with variance ratios close to $\pi/3$, in agreement with theory---and shows much lower bias and variance than engression, particularly for large $d$. These results reinforce the benefits of multi-step RML over one-step engression in statistical efficiency.

\begin{table}%[h]
\centering
\caption{Biases and variances for Gaussian adjacency matrix estimation.}\label{tab:sem}
\smallskip
\begin{tabular}{cc*{3}{cc}}
\toprule
\multicolumn{2}{c}{} &
\multicolumn{2}{c}{$d=5$} &
\multicolumn{2}{c}{$d=10$} &
\multicolumn{2}{c}{$d=20$} \\
$n$ & method & bias$^2$ & variance & bias$^2$ & variance & bias$^2$ & variance \\
\midrule
\multirow{3}{*}{100}
  & engression & 0.0793 & 1.9641 & 8.7834 & 23.7303 & 66.0000 & 164.940 \\
  & RML        & 0.0009 & 0.1304 & 0.0082 & 0.8465 & 0.0576  & 6.7168 \\
  & MLE        & 0.0010 & 0.1282 & 0.0084 & 0.8166 & 0.0537  & 6.4606 \\
\cmidrule{1-8}
\multirow{3}{*}{1000}
  & engression & 0.1387 & 2.1417 & 8.3178 & 22.2217 & 62.2733 & 164.000 \\
  & RML        & 0.0002 & 0.0128 & 0.0005 & 0.0787 & 0.0064  & 0.5771 \\
  & MLE        & 0.0002 & 0.0120 & 0.0005 & 0.0749 & 0.0049  & 0.5586 \\
\cmidrule{1-8}
\multirow{3}{*}{10000}
  & engression & 0.0967 & 2.2687 & 7.4485 & 23.4347 & 65.2973 & 164.451 \\
  & RML        & 0.0000 & 0.0013 & 0.0001 & 0.0081 & 0.0005  & 0.0576 \\
  & MLE        & 0.0000 & 0.0013 & 0.0000 & 0.0080 & 0.0004  & 0.0554 \\
\bottomrule
\end{tabular}
\end{table}

\begin{figure}
\centering
\begin{tabular}{ccc}
	\includegraphics[width=0.3\textwidth]{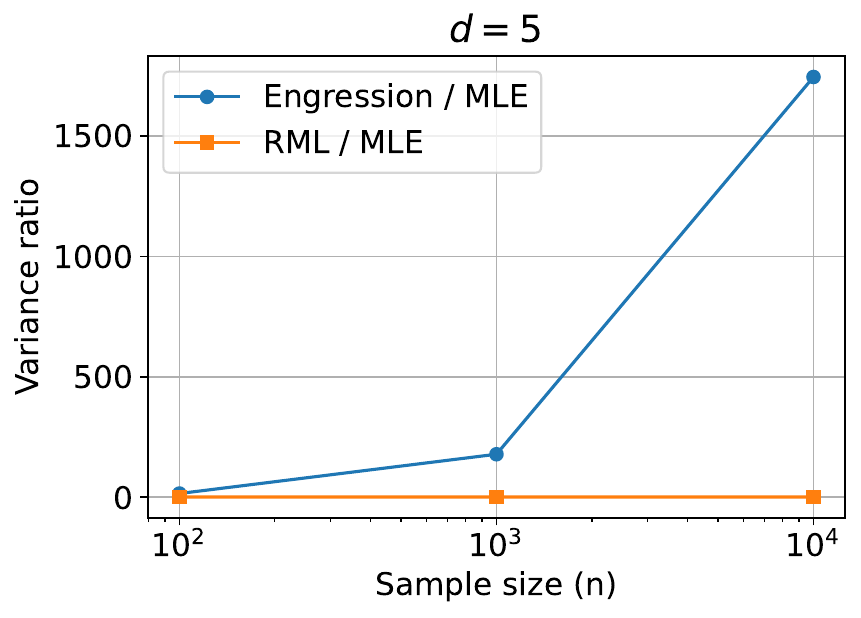} &
	\includegraphics[width=0.3\textwidth]{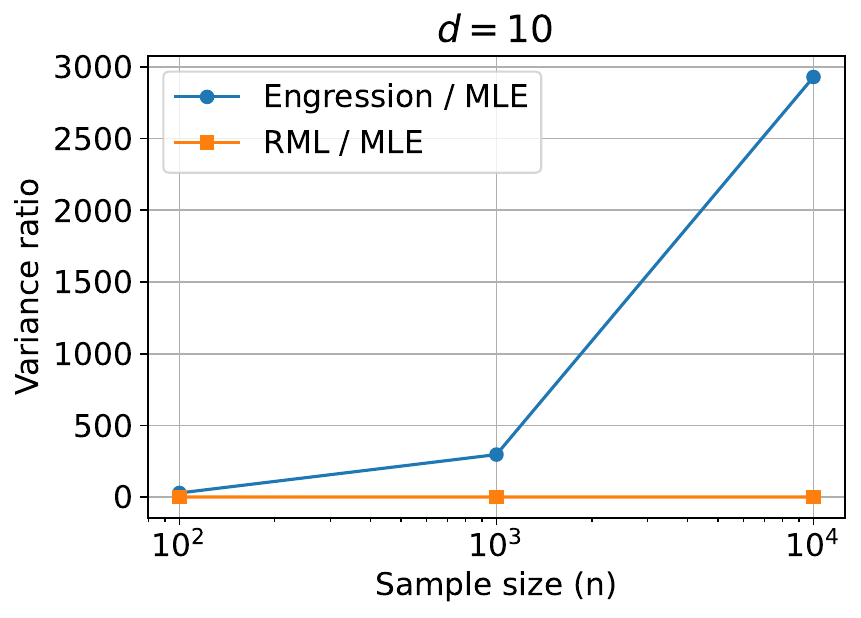} &
	\includegraphics[width=0.3\textwidth]{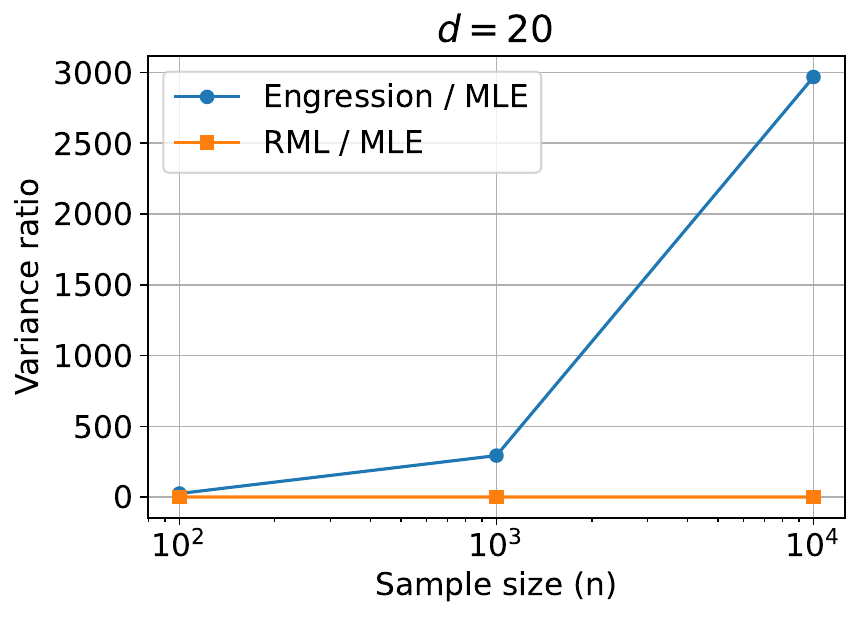} 
\end{tabular}\vspace{-0.1in}
\caption{Ratio of variances for Gaussian adjacency matrix estimation.}
\label{fig:sem_ratio}
\end{figure}

\medskip
\noindent\textbf{Hypothesis testing.}\quad
Inspired by the gain of RML in estimation efficiency, we explore enhancing a statistical distance with multiple conditional distances to improve the power in two-sample hypothesis testing. Consider two probability distributions $X\sim p$ and $W\sim q$ on $\mathbb{R}^d$, and the null hypothesis
 \[ H_0: p=q. \]
 A natural test statistic is based on the energy distance $D(p,q)$ defined in \eqref{eq:energy_distance}. To enhance it in an RML-fashion, we define forward processes for both distributions, i.e., $\{X_t\}$ and $\{W_t\}$, $t=0,1,\dots,T$, where $X_0\sim p$, $W_0\sim q$, and $X_T$ and $W_T$ follow a known distribution. Let $p_t(\cdot | x_t)$ denote the conditional distribution of $X_{t-1} | X_t$ and $q_t(\cdot | w_t)$ that of $W_{t-1} | W_t$. The RML-enhanced energy distance is defined as
 \begin{equation}\label{eq:rml_energy_distance}
 	\frac1T\sum_{t=1}^T \E\big[D(p_t(\cdot|X_t), q_t(\cdot|X_t))\big].
 \end{equation}
 
We provide an illustrative example in the linear structural equation case while we encourage a more generic study in the future. We consider distribution $p$ induced by model \eqref{eq:linscm} with adjacency matrix $A$ and $q$ induced by the same model with adjacency matrix $B$. For the RML-enhanced statistic, we use the forward process in \eqref{eq:sem_forward}. We consider $d=5$ and $n=200$ with $A$ and $B$ randomly generated. To maintain a $95\%$ significance level, we take the $95\%$ quantile of the test statistic under the null as the critical value, which is estimated from 1k simulated null datasets. We repeat the two testing procedures for 1k times on randomly simulated data.

Figure~\ref{fig:test} presents the test statistics based on the energy distance and the RML-enhanced version, along with the resulting power for varying Frobenius norms of $A-B$, which measure the closeness of $p$ and $q$. We observe that when the two distributions are fairly close, the energy distance-based test statistic shows almost no power, whereas the RML-enhanced statistic achieves substantial power as long as the distance exceeds a small threshold ($>0.01$). 
These results suggest the potential of the multi-step enhancement of RML not only for improving estimation efficiency but also for more powerful hypothesis tests. 
 
\begin{figure}
\centering
\begin{tabular}{@{}cc}
	\includegraphics[height=0.35\textwidth]{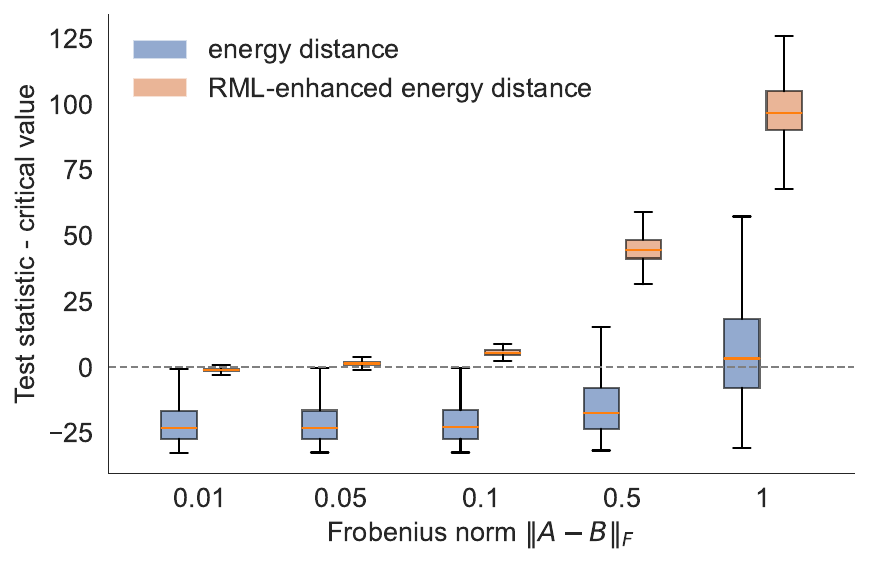} &
	\includegraphics[height=0.35\textwidth]{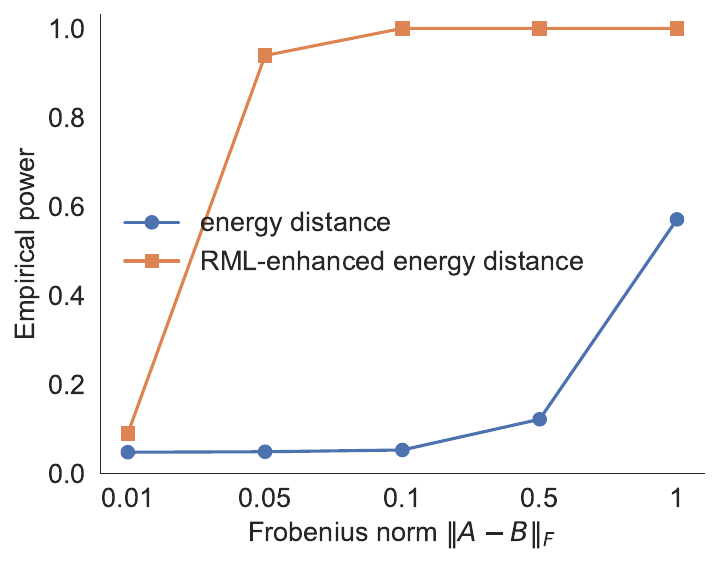} 
\end{tabular}\vspace{-0.1in}
\caption{Test statistics and power for two-sample testing.}
\label{fig:test}
\end{figure}

\subsection{Comparing Forward Schemes in Gaussian Mixture Model}\label{sec:forward_analysis}

In order to compare different forward schemes, we consider a simple Gaussian mixture model in 1D, as illustrated in Figure~\ref{fig:gmm-example}. 
In this example, we let \(s\in\{-1,1\}\) with \(\Pr(s=\pm1)=\frac12\), and
\[
X_{0}\mid s \;\sim\;\mathcal{N}\bigl(s,\sigma^2\bigr).
\]

\begin{figure}
\centering
\begin{tikzpicture}
  \begin{axis}[
      width=10cm, height=6cm,
      domain=-3:3, samples=250,
      xlabel={$x$}, ylabel={$p_{X_0}(x)$},
      axis lines=left,
      ymin=0, enlargelimits=false,
      tick label style={font=\small},
      every axis label/.style={font=\small},
      legend style={legend cell align=left, 
      legend pos=north east, xshift=1cm}
    ]
    
    \addplot+[smooth,black,mark=none]
      {0.5/(sqrt(2*pi)*\gmmsigma)*exp(-((x + 1)^2)/(2*\gmmsigma*\gmmsigma))
       +0.5/(sqrt(2*pi)*\gmmsigma)*exp(-((x - 1)^2)/(2*\gmmsigma*\gmmsigma))};

  \end{axis}
\end{tikzpicture}\vspace{-0.1in}
\caption{Simple 1-Dimensional GMM Example with $\sigma=\gmmsigma$.}
\label{fig:gmm-example}
\end{figure}

For \(t\in[0,1]\) we consider the following three cases for the forward process:
\[ X_t = (1- \frac t T) X_0 +  \varepsilon_t \qquad \mbox{ for  } t=0,\ldots, T\]
where 
\begin{align*}
\mbox{flow matching:} \quad & \varepsilon_t = \frac t T \eta_T \mbox{  and  }  \eta_T \sim\mathcal{N}(0, 1)  \\
\mbox{diffusion:} \quad & \varepsilon_0=0,\quad \varepsilon_t = \varepsilon_{t-1} + \frac{\sqrt{2t-1}}{T}\eta_t , \mbox{ and  } \eta_t \stackrel{iid} {\sim} \mathcal{N}(0,  1)\mbox{ for } t=1,\ldots, T \\
\mbox{X process:} \quad & \varepsilon_t =  \frac{t}{T}
\eta_t \mbox{ and  }  \eta_t\stackrel{iid} {\sim} \mathcal{N}(0, 1) \mbox{ for } t=1,\ldots, T 
\end{align*}

The choice makes sure that marginals of $X_t$ for each $t$ are identical.
It is easy to check that the conditional distribution of the reverse process is a mixture of Gaussian:
\[
X_{t-1} \mid X_t = x \sim \sum_{s \in \{-1, 1\}} w_s(x)\, \mathcal{N}\left( a x + b_s,\, \tau^2 \right), 
\]
where for $s = \pm 1$,
\[
w_s(x)
= \frac{1}{
1 + \exp\left(
- \dfrac{2 s x (1 - \frac{t}{T})}{
(1 - \frac{t}{T})^2 \sigma^2 + \left( \frac{t}{T} \right)^2
}
\right)
}, \text{ and }
b_s = s \left(1 - \frac{t-1}{T} - a \left(1 - \frac{t}{T} \right) \right).
\]
%and
%\[
%b_s = s \left(1 - \frac{t-1}{T} - a \left(1 - \frac{t}{T} \right) \right).
%\]
Moreover $a$ and $\tau^2$ are different for different methods:
\begin{align*}
\mbox{flow matching:} \quad &   
\begin{cases}
a = \frac{
(1 - \frac{t-1}{T})(1 - \frac{t}{T}) \sigma^2 + \frac{(t-1)t}{T^2}
}{
(1 - \frac{t}{T})^2 \sigma^2 + \frac{t^2}{T^2}
} \\
\tau^2 =
\left(1 - \frac{t-1}{T}\right)^2 \sigma^2 + \frac{(t-1)^2}{T^2}
-
\frac{
\left(
(1 - \frac{t-1}{T})(1 - \frac{t}{T}) \sigma^2 + \frac{(t-1)t}{T^2}
\right)^2
}{
(1 - \frac{t}{T})^2 \sigma^2 + \frac{t^2}{T^2}
}.
\end{cases}
\\
\mbox{diffusion:} \quad & 
\begin{cases}
a = \frac{
\left(1 - \frac{t-1}{T} \right) \left(1 - \frac{t}{T} \right) \sigma^2 + \frac{(t-1)^2}{T^2}
}{
\left(1 - \frac{t}{T} \right)^2 \sigma^2 + \frac{t^2}{T^2}
}\\
\tau^2 =
\left(1 - \frac{t-1}{T} \right)^2 \sigma^2 + \frac{(t-1)^2}{T^2}
-
\frac{
\left(
\left(1 - \frac{t-1}{T} \right)\left(1 - \frac{t}{T} \right) \sigma^2 + \frac{(t-1)^2}{T^2}
\right)^2
}{
\left(1 - \frac{t}{T} \right)^2 \sigma^2 + \frac{t^2}{T^2}
}.
\end{cases}\\
\mbox{X process:} \quad &  
\begin{cases}
a = \frac{
\left(1 - \frac{t-1}{T} \right)\left(1 - \frac{t}{T} \right) \sigma^2
}{
\left(1 - \frac{t}{T} \right)^2 \sigma^2 + \left( \frac{t}{T} \right)^2
}\\
\tau^2 =
\left(1 - \frac{t-1}{T} \right)^2 \sigma^2 + \left( \frac{t-1}{T} \right)^2
-
\frac{
\left(
\left(1 - \frac{t-1}{T} \right)\left(1 - \frac{t}{T} \right) \sigma^2
\right)^2
}{
\left(1 - \frac{t}{T} \right)^2 \sigma^2 + \left( \frac{t}{T} \right)^2
}.
\end{cases}
\end{align*}

\begin{figure}[ht]
\centering
 \input{parts/gmm-example-reverse.tex}\vspace{-0.1in}
  \caption{Conditional densities \(p_{X_{t-1}\mid X_t=\gmmx}(y)\)  
           (\(\sigma=\gmmsigma,\; t=\gmmt,\; T=\gmmT\)).}
  \label{fig:conditional-densities}
\end{figure}

We are particularly interested in the case of $\sigma \to 0$,
because this is the case that the original distribution is difficult to generate due to separation of the two modalities.
In this case, we have 
\begin{align*}
\mbox{flow matching:} \quad & a \to \frac{t-1}{t},\quad \tau^2 \to 0\\
\mbox{diffusion:} \quad & a \to \frac{(t-1)^2}{t^2},\quad \tau^2 \to \frac{(2t-1)}{t^2}\frac{(t-1)^2}{ T^2} \\
\mbox{X process:} \quad & a \to 0,\quad \tau^2 \to \frac{(t-1)^2}{T^2} .
\end{align*}

We note that the smaller $a$ is, and the larger $\tau^2$ is, the less sensitive the conditional distribution is with respect to $x$ and $y$. Based on the computation, it follows that as $\sigma \to 0$, the reverse conditional probability of the $X$ process is smoother than that of the diffusion process, than that of the flow matching process. See Figure~\ref{fig:conditional-densities} for an illustration. 

To further illustrate this in the context of RML training, we train RML models with three different forward processes for the Gaussian mixture example in Figure~\ref{fig:mog}. All previous numerical results were obtained with the X process. Figure~\ref{fig:mog_forward_process} shows generated samples with 10 time steps. RML with the X process produces samples that best match the true distribution, while the other two processes tend to place samples in low-density regions.

 These results highlight that the choice of forward process is critical in multi-step generative modeling. In particular, the X process can facilitate learning by inducing smoother reverse conditionals. The flexibility of RML to accommodate a broad range of forward processes is therefore a key advantage over diffusion models which allow only limited forward-process choices.

 \begin{figure}
 \centering
 \begin{tabular}{ccc}
 	\includegraphics[width=0.25\textwidth]{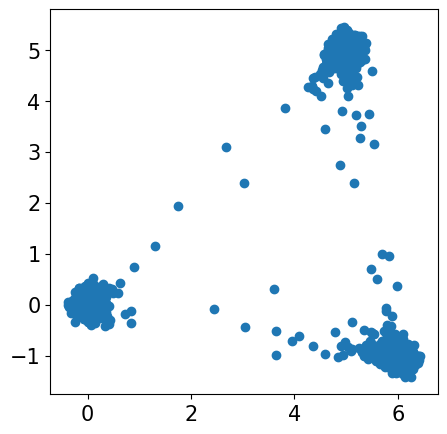} &
 	\includegraphics[width=0.25\textwidth]{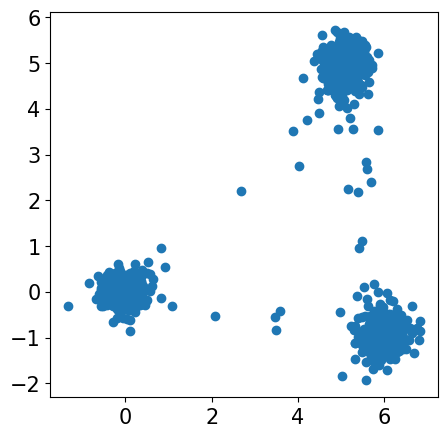} &
 	\includegraphics[width=0.25\textwidth]{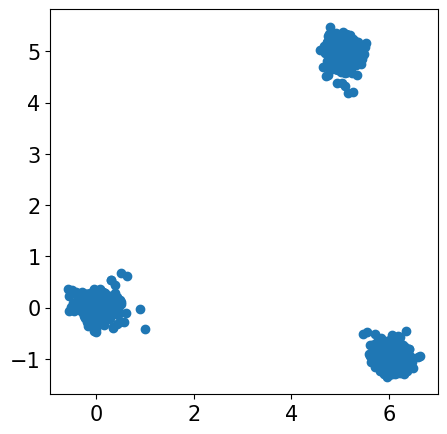}\\
 	\small flow matching & \small diffusion & \small X process
 \end{tabular}
 \caption{Samples generated from RML with different forward processes for the example in Figure~\ref{fig:mog}.}
 \label{fig:mog_forward_process}
 \end{figure}

\section{Application to Regional Precipitation Prediction}\label{sec:precip}
We consider the problem of climate data prediction for monthly precipitations in central Europe at a spatial resolution of $128\times128$. Our goal is to model the full distribution of precipitation, not just its mean or median, because extreme rainfalls can cause severe floods or droughts and widespread damage. Note that to predict regional climate, in practice, one typically has extra information, such as coarse-scale climate, that provides more signal. Here, however, as a proof of concept, we address the more challenging unconditional setting without covariates, in which the target distribution has higher variance and can be harder to learn.

We apply Reverse Markov Learning to estimate the precipitation distribution via sampling. We define the forward path via average pooling as described in Example III in Section~\ref{sec:forward}. As illustrated in Figure~\ref{fig:precip_true}, starting from the original $128\times128$ spatial map, we apply $2\times2$ average pooling to downsample to $64\times64$, reducing the total dimensionality by a factor of 4, and repeat this process until reaching a $2\times2$ resolution. The final step maps to a standard Gaussian. We adopt three kernel sizes---2, 4, and 8---yielding 7, 4, and 3 time steps, respectively. This approach reduces dimensionality through fixed, interpretable maps without learning an additional encoder. In the reverse Markov process, each step learns to reconstruct the high-resolution precipitation field from its low-resolution counterpart---a task that is itself scientifically meaningful and related to (though simpler than) statistical downscaling~\citep{wilby1998statistical,wilby2013statistical}.

\begin{figure}
    \centering
    \begin{tabular}{@{}c@{}c@{}c@{}c@{}c@{}c@{}c@{}c@{}}
        $t=0$ & $t=1$ & $t=2$ & $t=3$ & $t=4$ & $t=5$ & $t=6$ & $t=7$\\
        \includegraphics[width=0.12\textwidth]{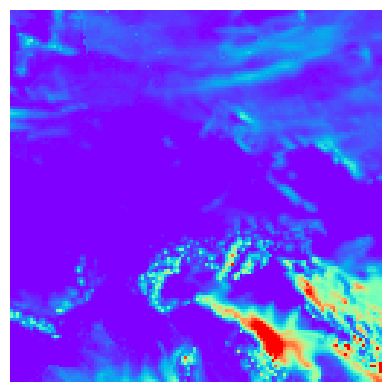}  & 
        \includegraphics[width=0.12\textwidth]{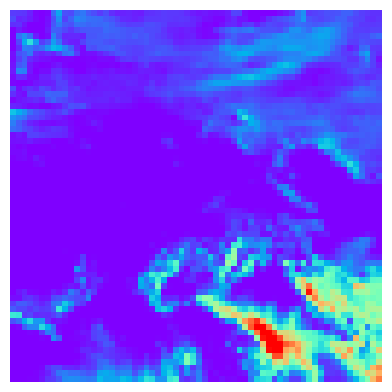} & 
        \includegraphics[width=0.12\textwidth]{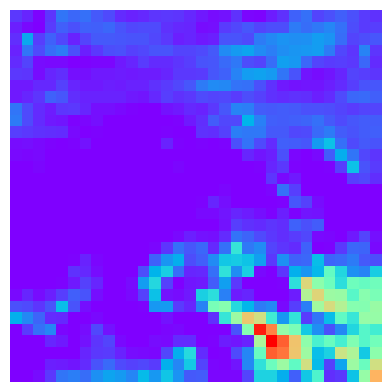} & 
        \includegraphics[width=0.12\textwidth]{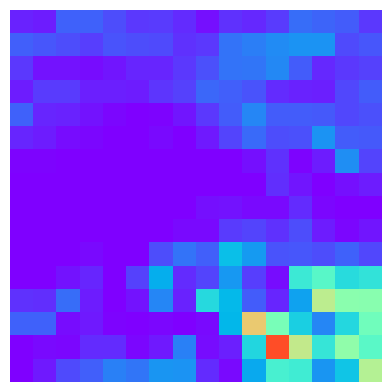} & 
        \includegraphics[width=0.12\textwidth]{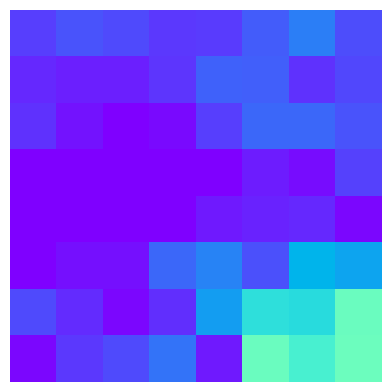} & 
        \includegraphics[width=0.12\textwidth]{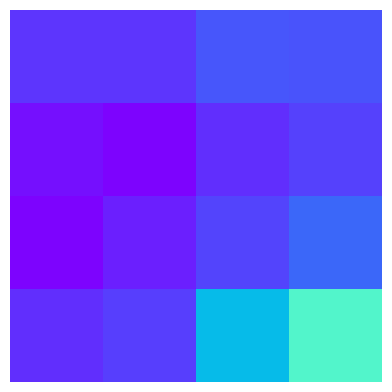} & 
        \includegraphics[width=0.12\textwidth]{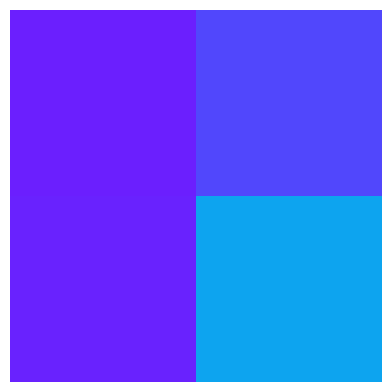} & 
        \includegraphics[width=0.12\textwidth]{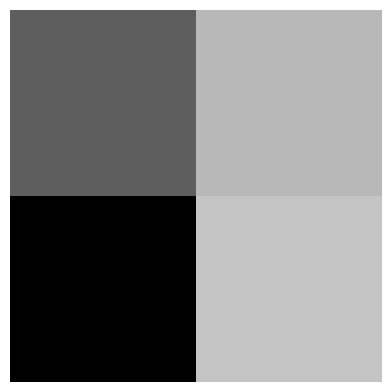}\vspace{-5pt} \\
        \small{$128\times128$} & \small{$64\times64$} & \small{$32\times32$} & \small{$16\times16$} & \small{$8\times8$} & \small{$4\times4$} & \small{$2\times2$} & \small{Gaussian}\\
    \end{tabular}
    \caption{Precipitation data and the average pooling forward process at a factor of $2^2$.}
    \label{fig:precip_true}
\end{figure}

\begin{figure}
    \centering
    \begin{tabular}{@{}c@{}c@{}c@{}}
        \includegraphics[height=0.46\textwidth]{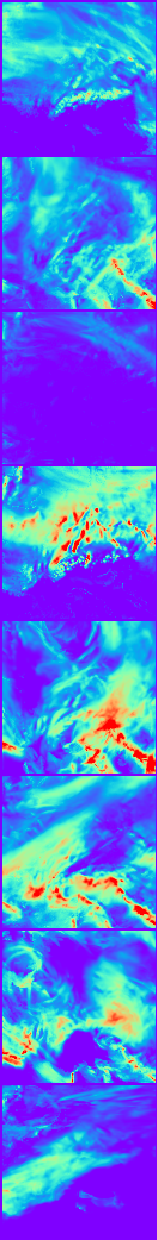}\hspace{4pt} &
        \includegraphics[height=0.46\textwidth]{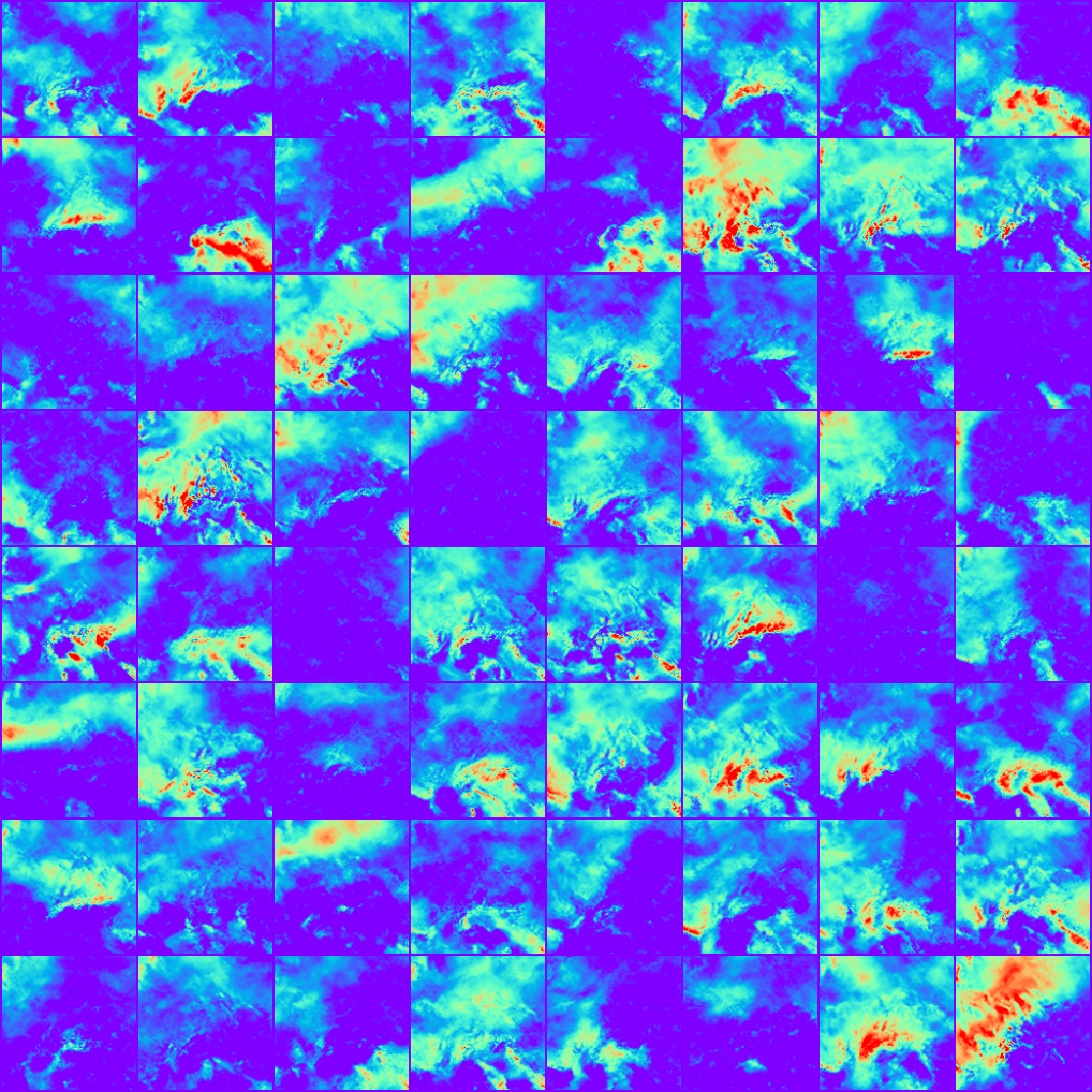}\hspace{4pt} & 
        \includegraphics[height=0.46\textwidth]{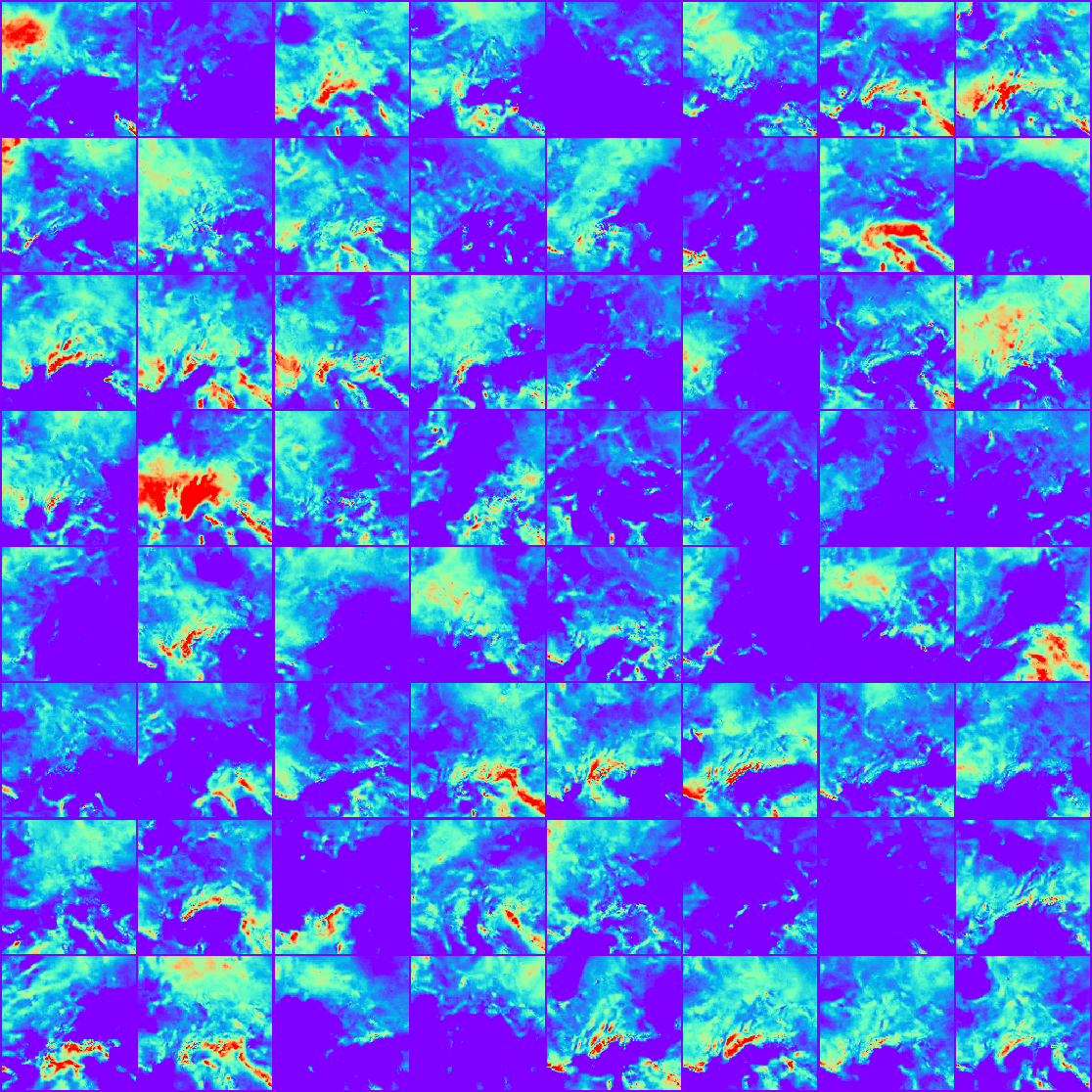} \\
        Truth & Factor of $2^2$ ($T=7$) & Factor of $8^2$ ($T=3$)
    \end{tabular}
    \caption{Generated samples from Reverse Markov Learning with $T=7$ or $T=3$ steps, in comparison to real data. }
    \label{fig:precip_gen}
\end{figure}

Figure~\ref{fig:precip_gen} shows the generated samples from RML with $T=7$ steps (average pooling with a kernel size of 2) and $T=3$ steps (kernel size of 8) in comparison to real data. We can see that samples generated from RML with too few steps ($T=3$) appear patchy and unrealistic, significantly worse than samples obtained from RML with more steps ($T=7$), which visually close to real data. %Figure~\ref{fig:precip_all_t} shows the generated samples at each step and illustrates how the distributions of the reverse Markov process behave in this case. 

Furthermore, we consider several quantitative metrics. To assess the joint distribution of the entire precipitation field, we use the energy distance, defined as \eqref{eq:energy_distance}, between the true and estimated distributions. We also compute the average and maximum of marginal energy distances and Wasserstein distances across each location. In addition, we consider the rank histogram which is a common tool for evaluating probabilistic forecasts by repeatedly tallying the rank of the true observation relative to values from generated samples sorted from lowest to highest~\citep{hamill2001interpretation}. When the probabilistic forecasts are well calibrated, the rank histogram should be as flat as a uniform distribution. Figure~\ref{fig:rankhist} presents the rank histograms, where the histogram of RML samples with $T=7$ looks more uniform (flat) than that of RML samples with $T=3$. For a quantitative metric, we compute the TV distance between the rank histogram and a uniform distribution, which we call the rank histogram TV distance.  

Figure~\ref{fig:precip_metric} summarizes all the metrics as a function of the number of steps. We can see that RML with a larger number of steps consistently performs the best in all metrics. Despite being the largest one, it still only takes 7 steps in total, which is much smaller than what is typically adopted in diffusion models. Besides, we start from a very low dimensional space, which also reduces computational costs. 
% For a comprehensive evaluation of marginal distributions, we consider rank histogram of 

\begin{figure}
\centering
\begin{tabular}{ccc}
	\footnotesize{Joint energy distance} & \footnotesize{Average marginal energy distances} & \footnotesize{Max marginal energy distance}\vspace{-2pt}\\
	\includegraphics[width=0.3\textwidth]{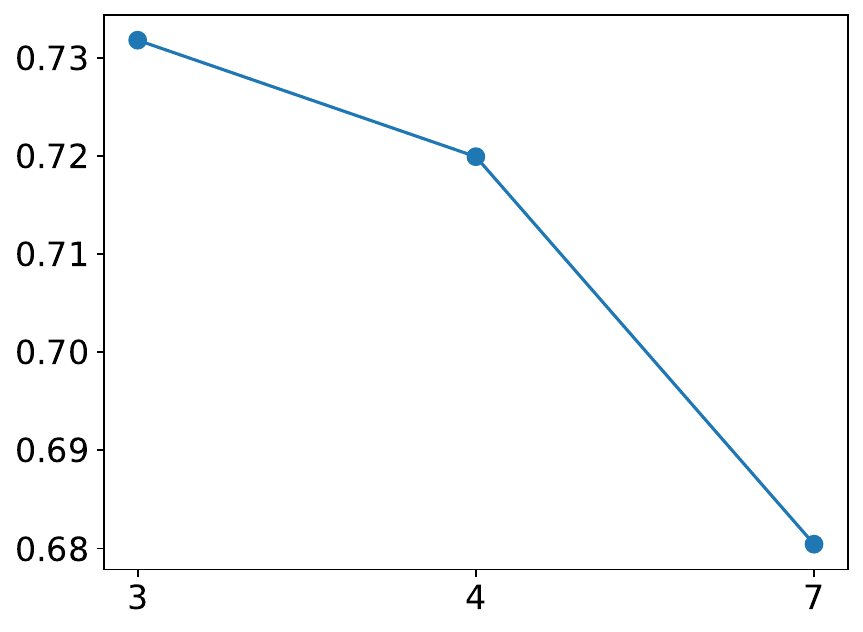} &
	\includegraphics[width=0.3\textwidth]{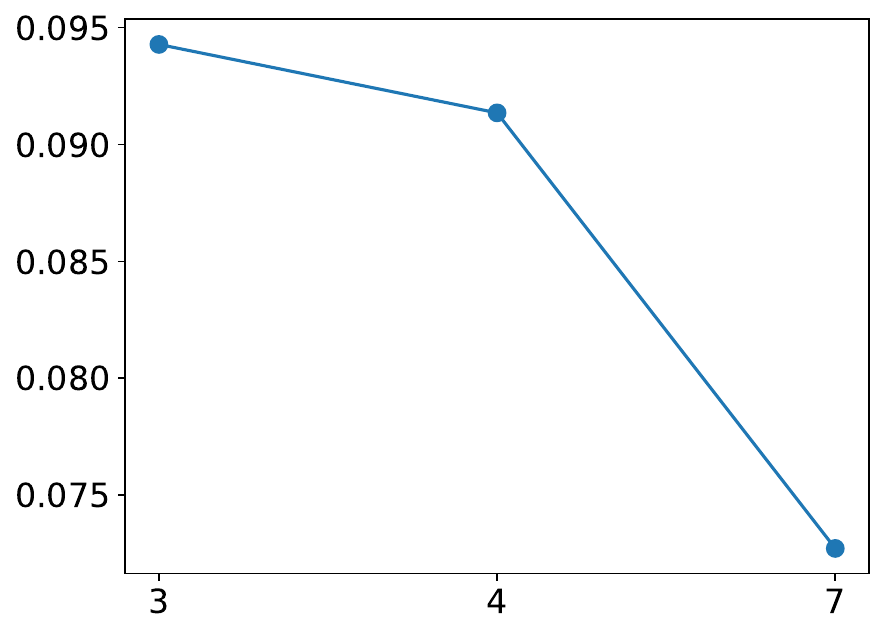} &
	\includegraphics[width=0.3\textwidth]{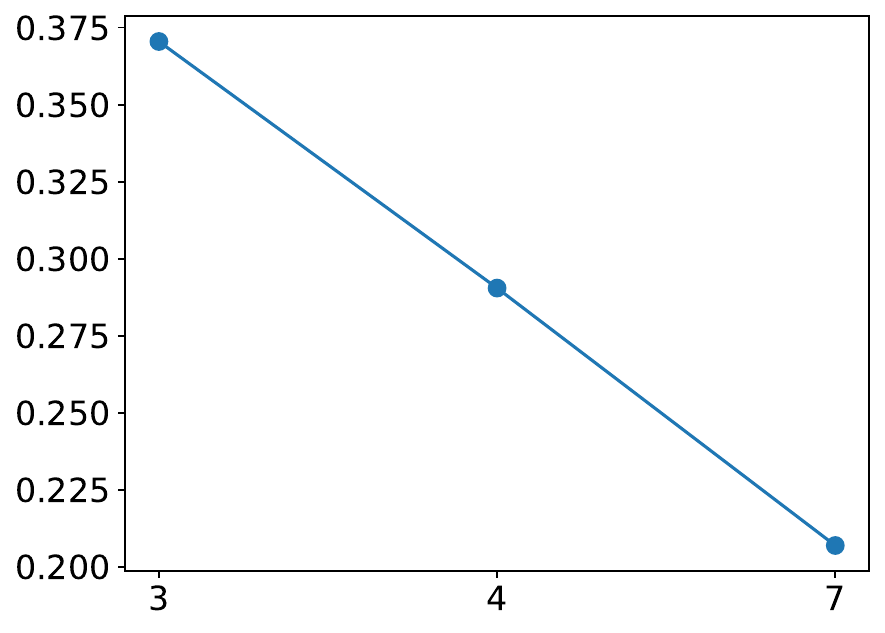}  \\
	\footnotesize{Rank histogram TV distance} & \footnotesize{Average marginal Wasserstein distances} & \footnotesize{Max marginal Wasserstein distance}\vspace{-2pt}\\
	\includegraphics[width=0.3\textwidth]{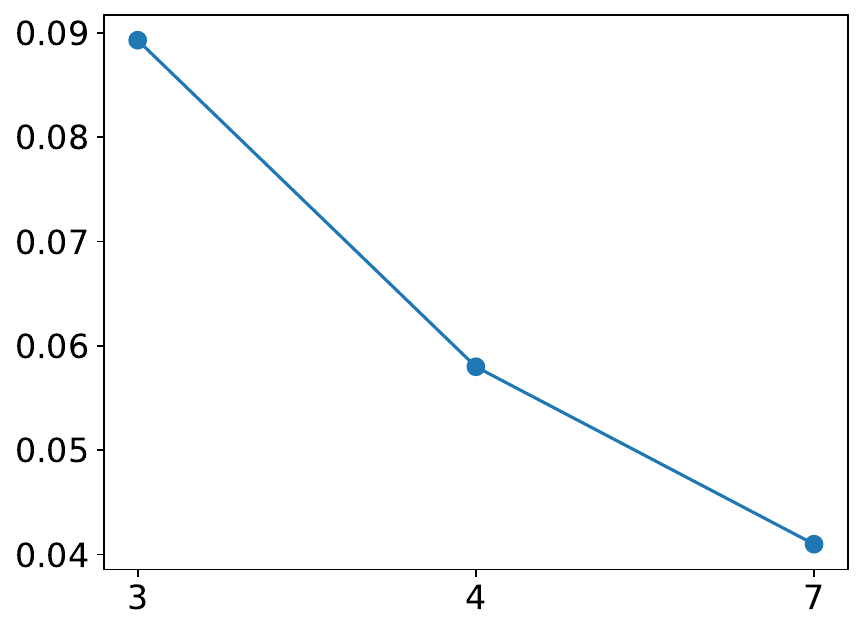} & 
	\includegraphics[width=0.3\textwidth]{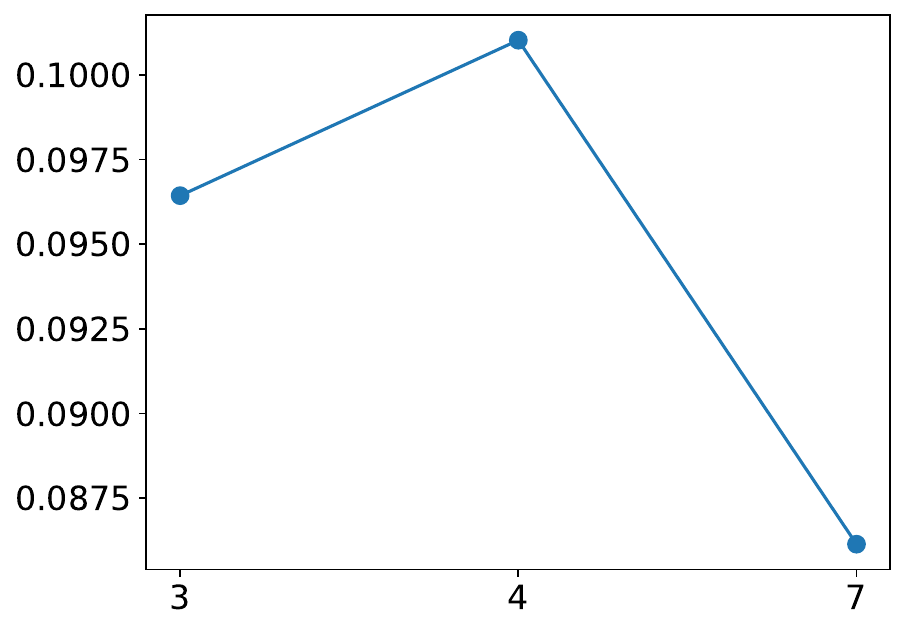} &
	\includegraphics[width=0.3\textwidth]{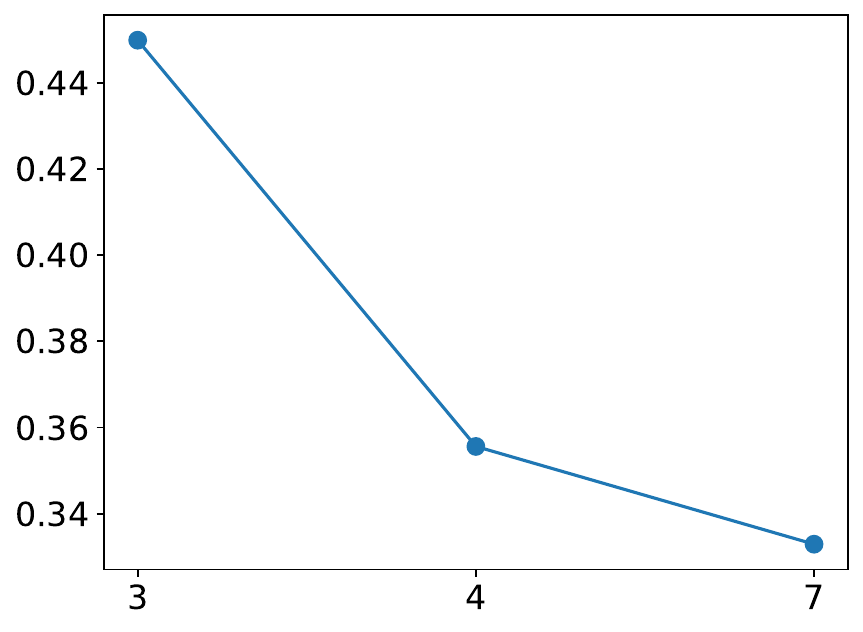}    
\end{tabular}
\caption{Quantitative metrics for learning the regional precipitation distribution as a function of the number of steps in the RML method.}\label{fig:precip_metric}
\end{figure}

\begin{figure}
	\centering
	\begin{tabular}{cc}
		\includegraphics[width=0.35\textwidth]{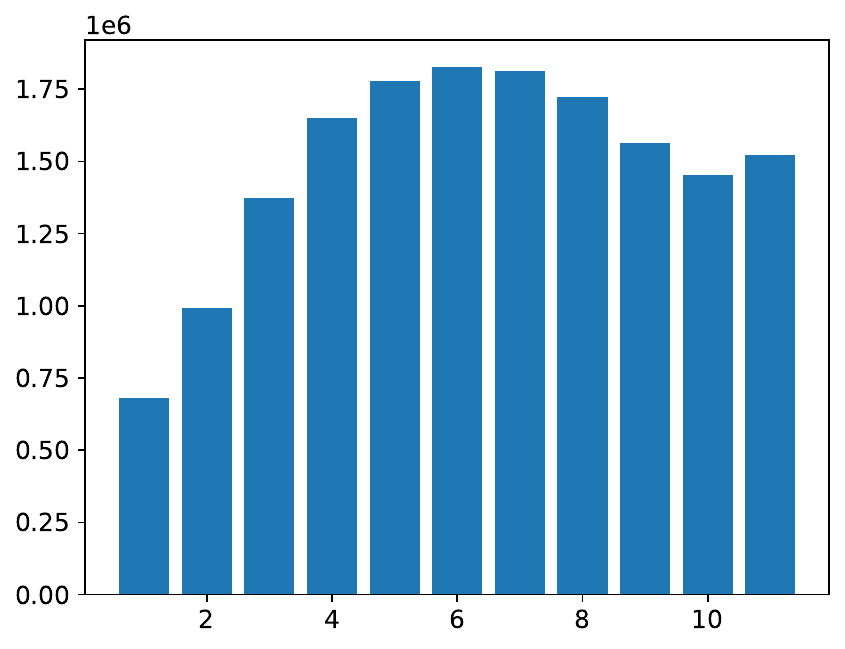} \hspace{4pt}&\hspace{4pt}
		\includegraphics[width=0.35\textwidth]{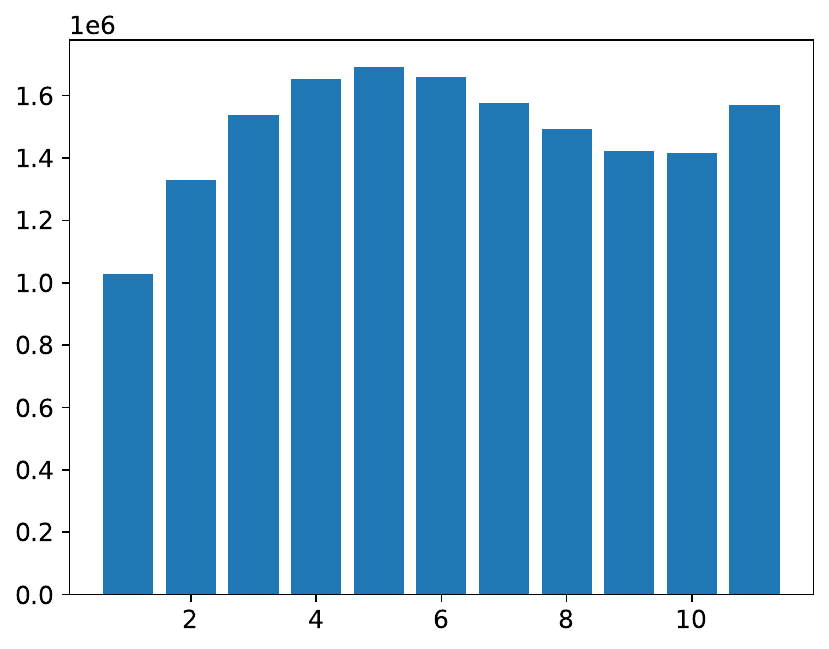}\\
		$T=3$ \hspace{4pt}&\hspace{4pt} $T=7$
	\end{tabular}
	\caption{Rank histograms of RML samples.}\label{fig:rankhist}
\end{figure}

\section{Conclusion and Outlook}
\label{sec:conclusion}
We introduced Reverse Markov Learning (RML), a new framework for generative modeling of complex distributions via a multi-step generative process. Extending engression, RML defines a forward process from data to noise and learns the reverse Markov process that incrementally reconstructs the target distribution from the known distribution.

Key advantages of RML include:
\begin{itemize}\vspace{-3pt}
	\itemsep 2pt
    \item Through multiple steps, RML can learn complex distributions better than one-step generative models such as engression.
    \item RML accommodates more general forward processes than those in diffusion models or flow matching, allowing for pursuing ``optimal" choices in terms of statistical or computational efficiency, or flexibly tailoring to various scientific applications.
    \item RML supports dimension reduction in the forward process, enabling principled low-dimensional diffusion for substantial computational savings in training and inference without losing distributional fidelity, especially in very high dimensions.
    \item The finite-step formulation avoids any discretization error even with relatively few steps, making sampling more efficient than in diffusion models.
\end{itemize}

From a theoretical standpoint, we established the correctness of RML by proving that the proposed reverse process reconstructs the target distribution. We also demonstrated that RML generalizes diffusion models and flow matching, recovering the latter as the continuous limit in a special case. Our statistical analysis ensures that the error occurred in the samples from RML is well controlled even if we encounter errors at each step of RML training. Through specific examples, we further investigated concretely the benefits of RML in estimation efficiency and the statistical impact of the forward-process choice. 

Most theoretical results were illustrated with a mixture-of-Gaussians example. In a challenging real-world task---regional precipitation prediction---RML effectively produced realistic precipitation fields, confirmed by both visual inspection and scientifically meaningful quantitative metrics. These findings underscore the promise of RML for high-dimensional generative modeling.

Looking ahead, scaling up the RML framework in both model and data size would be of great value for pushing its empirical limits. Moreover, RML's flexibility of the forward process invites creative new designs with both theoretical and practical payoffs. %leveraging the flexibility of RML in its forward process, it is worth exploring creative new designs of the forward process and examining their theoretical and practical implications. 
Finally, %we sketched an idea in Section~\ref{sec:linsem} of enhancing a statistical distance inspired by RML in the context of hypothesis testing. A more systematic study along this line could be interesting to strengthen a weak but easy-to-estimate statistical distance, such as the energy distance. 
the RML-inspired enhancement of statistical distances (as sketched in Section~\ref{sec:linsem}) could open an interesting direction for strengthening weak but easy-to-estimate metrics like the energy distance, with potential impact on hypothesis testing.
% Our empirical studies validate the effectiveness of RML in both synthetic and real-world scenarios. In a simple illustrative example with Gaussian mixtures, RML successfully captured the true distribution with a small number of steps, avoiding the oversmoothing issue observed in one-step engression. In a more challenging application to regional precipitation prediction, RML effectively generated realistic precipitation fields, as demonstrated in both visualizations and quantitative metrics. These results highlight the advantages of our framework in learning high-dimensional distributions with structured dependencies.

% Overall, RML provides a promising approach to generative modeling with several advantages: (i) it is more capable of learning complex distributions than one-step generative models; (ii) it allows for general forward processes, including those with dimension reduction, thus offering a flexible framework that can be tailored to various applications; (iii) it is computationally more efficient than diffusion models, as it avoids unnecessary high-dimensional computations and naturally discretizes the generative process in finite steps. %; and (iii) it offers a flexible framework that can be tailored to various applications, including scientific and real-world data modeling.

\bibliography{ref.bib}
\bibliographystyle{apalike}

\newpage
\appendix
\section{Proofs}
\begin{proof}[Proof of Theorem~\ref{thm:rmp}]
Let $\tilde{p}_t(\tilde{x}_t|y)$ be the marginal distribution of 
$\tilde{X}_t$, conditioned on $Y=y$. 
Let $p^*_t(x_t|Y=y)$ be the marginal distribution of 
$X_t$, conditioned on $Y=y$. By the assumption of the forward process, we have
\[
\tilde{p}_T(x_T|y)=p^*_T(x_T|y)=q^*(x_T|y) .
\]
Assume that the statement of
\[
\tilde{p}_t(x_t|y)=p^*_t(x_t|y)
\]
 is true at any $1 \leq t \leq T$. Then
by Algorithm~\ref{alg:rms}, we have
\[
\tilde{p}_t(x_{t-1},x_t|y)
= p_t^*(x_{t-1}|x_t,y) \tilde{p}_t(x_t|y)
= p_t^*(x_{t-1}|x_t,y) p^*_t(x_t|y)
= p^*_t(x_{t-1},x_t|y) .
\]
This means that the theorem holds at $t$. 
By taking marginal at $x_{t-1}$, this implies also that
\[
\tilde{p}_{t-1}(x_{t-1}|y)=p^*_{t-1}(x_{t-1}|y) .
\]
Now we obtain the desired result by induction on $t$.
\end{proof}

\begin{proof}[Proof of Theorem~\ref{thm:continuous-discretized}]
We know that a sample of $p^*_t(X_{t-1},X_t|y)$ is given by
\[
  h(Z,y,(t-1)/T) , h(Z,y,t/T) : \qquad Z \sim \cP_{Z|Y}(\cdot|Y=y) .
\]
We know that
\[
h(Z,y,(t-1)/T) = \tilde{x}_t - (1/T) \frac{\partial}{\partial s} h(Z,y,t/T) + o(1/T) , \qquad \tilde{x}_t=h(Z,y,t/T) .
\]
Therefore the mean of the reverse Markov distribution $p^*_t(X_{t-1}|X_t=\tilde{x}_t,y)$ in
Algorithm~\ref{alg:rms} is given by
\[
\tilde{x}_t - (1/T) g(\tilde{x}_t,y,t/T) + o(1/T) .
\]
The variance of $p^*_t(X_{t-1}|X_t=\tilde{x}_t,y)$ is $O(1/T^2)$ by using the uniform bounded variance assumption. 
The aggregated variance of $p^*_t(X_{0}|X_t=\tilde{x}_t,y)$ is no more than $O(1/T)$, 
which means that the effect of variance vanishes (in terms of Wasserstein distance)  as $T \to \infty$.
It follows that as $T \to \infty$, the reverse Markov distribution $p^*_t(X_{t-1}|X_t=\tilde{x}_t,y)$
becomes deterministic, which can be characterized by a shift in the mean:
\[
\tilde{x}_t - (1/T) g(\tilde{x}_t,y,t/T) + o(1/T) .
\]
Since the error term $o(1/T)$ does not affect the convergence in Wasserstein distance as $T \to \infty$, we obtain the desired result from Theorem~\ref{thm:rmp}. 
\end{proof}

\begin{proof}[Proof of Lemma~\ref{lem:func-compose}]
Let $\Pi_t^y(x',x)$ be an optimal coupling between measures $\hat{q}_{s+1}(x'|y)$ and $q_{s+1}^*(x|y)$ which achieve the $W_2(\hat{q}_{s+1}(\cdot|y),q_{s+1}^*(\cdot|y))$. 
We have
\begin{align*}
u_s=&\left(\E_{y \sim p^*} W_2(\hat{p}_s(\cdot|y), p_s^*(\cdot|y))^2\right)^{1/2}\\
\leq& \left(\E_{y \sim p^*} \E_{x \sim p_{s+1}^*} W_2(\hat{p}_s(\cdot|x,y), p_s^*(\cdot|x,y))^2\right)^{1/2} 
 + \left(\E_{y \sim p^*}  \E_{(x'x) \sim \Pi_t^y} W_2(\hat{p}_s(\cdot|x',y), \hat{p}_s(\cdot|x,y))^2\right)^{1/2}\\
 \leq&  \delta_t
 + \left(\E_{y \sim p^*}  \E_{(x'x) \sim \Pi_t^y} \E_{\varepsilon \sim \cN(0,I)} \|\hat{g}_s(x',y,\varepsilon)- \hat{g}_s(x,y,\varepsilon)\|_2^2\right)^{1/2} \\
 \leq&  \delta_{s+1}  + L_{s+1} \left(\E_{y \sim p_*}  \E_{(x'x) \sim \Pi_t^y}  \|x'-x\|_2^2\right)^{1/2} 
 = \delta_{s+1}  + L_{s+1} u_{s+1} .
\end{align*}
Using induction, we obtain the desired result. 
\end{proof}

\begin{proof}[Proof of Theorem~\ref{thm:alternating}]
Consider the true joint distribution $p^*(\tilde{X}_1,\ldots,\tilde{X}_T|y)=p^*(\tilde{X}_T|y)\prod_{t=1}^T p_t^*(\tilde{X}_{t-1}|\tilde{X}_t,y)$ of the reverse process, and the estimated distribution  $\hat{p}(\tilde{X}_1,\ldots,\tilde{X}_T|y)=p^*(\tilde{X}_T|y)\prod_{t=1}^T \hat{p}_t(\tilde{X}_{t-1}|\tilde{X}_t,y)$ of the reverse process.
Given $x \sim p_t^*(X_t=x|y)$, let $\tilde{p}_{t}(\cdot|x,y)$ be the distribution of 
$f_t(\tilde{X}_{s(t-1)}',\epsilon')|\tilde{X}_t=x,y$  obtained at step $t$ in Algorithm~\ref{alg:rmg-alt}.
Let $\tilde{p}_t^*(\cdot|x,y)$ the the distribution of $f_t(X_{s(t-1)}',\epsilon')|\tilde{X}_t=x,y$ where $X_{s(t-1)}$ is generated from Algorithm~\ref{alg:rmg-alt} using the true reverse process. 
Then using the same argument of Lemma~\ref{lem:func-compose}, we obtain
\[
\left(\E_{y \sim p^*} \E_{x \sim p_t^*(X_t=x|y)} W_2(p_t^*(X_{t-1}|X_t=x,y)||\tilde{p}_t(\cdot|x,y))^2\right)
\leq \alpha_{t-1} \left(\sum_{t'=s(t-1)+1}^{t-1} L_{s(t-1)}^{t'} \delta_{t'+1}\right) .
\]
Since $\tilde{X}_{t-1}$ and $X_{t-1}$ are generated from adding the same Gaussian noise to $\tilde{p}_t(\cdot|x,y)$ and $\tilde{p}_t^*(\cdot|x,y)$ respective, we know from Gaussian KL divergence formula that
\[
\E_{x \sim p_t^*} \KL(p_t^*(\cdot|x_t=x,y)||\hat{p}_t(\cdot|x_t=x,y)) \leq \frac{\alpha_{t-1}^2 \left(\sum_{t'=s(t-1)+1}^{t-1} L_{s(t-1)}^{t'} \delta_{t'+1}\right)^2}{2 \beta_{t-1}^2}  .
\]
Using Chain rule of KL-divergence, we obtain
\begin{align*}
&\E_{y \sim p_*} \KL(p_1^*(\cdot|y)||\hat{p}_1(\cdot|y)) 
\leq \E_{y \sim p_*} \KL(p^*(\tilde{X}_1,\ldots,\tilde{X}_T|y)||\hat{p}(X_1,\ldots,X_T|y)) 
\\
=& \sum_{t=2}^{T} \E_{y \sim p^*} \E_{x \sim p_t^*} \KL(p_t^*(\tilde{X}_{t-1}|\tilde{X}_{t}=x,y)||\hat{p}_t(X_{t-1}|X_t=x,y)) \\
\leq& \sum_{t=2}^{T} \frac{\alpha_{t-1}^2 \left(\sum_{t'=s(t-1)+1}^{t-1} L_{s(t-1)}^{t'} \delta_{t'+1}\right)^2}{2 \beta_{t-1}^2}  .
\end{align*}
This implies the result by shifting summation of $t$ from $2$ to $T$ by $1$: from $1$ to $T-1$. 
\end{proof}

\begin{proof}[Proof of Theorem~\ref{thm:asy_var}]
All three estimators are classical M-estimators with the regularity conditions met, yielding consistency and asymptotic normality \citep{van2000asymptotic}. To compute the asymptotic variance, we first consider the RML estimator. For each $k=2,\dots,d$, define $u(b_k)=X_k-b_k^\top X_{1:(k-1)}$ and $\rho(t)=\E_\varepsilon|t-\varepsilon|$ for $\varepsilon\sim\cN(0,1)$. Then $\ell_k(b_k;x_1,\dots,x_k)=\rho(u(b_k))$. Note that $\rho'(t)=\E[\mathrm{sign}|t-\varepsilon|]=2\Phi(t)-1$ and $\rho''(t)=2\phi(t)$ with $\Phi$ and $\phi$ being the cdf and pdf of the standard Gaussian. Compute the gradient and hessian of the loss as follows:
\begin{align*}
	\dot\ell(b_k;X_1,\dots,X_k) &= \nabla_{b_k}\rho(u(b_k)) = -(2\Phi(u(b_k))-1)X_{1:(k-1)}\\
	\ddot\ell(b_k;X_1,\dots,X_k) &= \nabla_{b_k}\dot\ell(b_k;X_1,\dots,X_k) = -2\phi(u(b_k))X_{1:(k-1)}X_{1:(k-1)}^\top.
\end{align*}
Let $\Sigma_{k-1}=\mathrm{Cov}(X_{1:(k-1)})$. Compute the pieces in the sandwich formula:
\begin{align*}
	A_k := \E[\ddot\ell(b^*_k;X_1,\dots,X_k)] &= -2\E[\phi(X_k-b^\top X_{1:(k-1)})X_{1:(k-1)}X_{1:(k-1)}^\top] \\
	&= -2\E[\phi(\varepsilon_k)]\mathrm{Cov}(X_{1:(k-1)}) = \sqrt{\frac1\pi}\Sigma_{k-1}
\end{align*}
and 
\begin{align*}
	B_k := \E[\dot\ell(b^*_k;X_1,\dots,X_k)\dot\ell(b^*_k;X_1,\dots,X_k)^\top] &= \E[(2\Phi(\varepsilon_k)-1)^2X_{1:(k-1)}X_{1:(k-1)}^\top] \\
	&= \E[(2\Phi(\varepsilon_k)-1)^2]\Sigma_{k-1}=\frac13\Sigma_{k-1}.
\end{align*}
Hence the asymptotic variance of $\hat{b}_k$ is given by
\begin{align*}
	V_k := A_k^{-1}B_kA_k^{-1} = \frac\pi3\Sigma_{k-1}^{-1}.
\end{align*}
Then the trace of the asymptotic variance of $\mathrm{vec}(\hat{B}_{\mathrm{RML}})$ is 
\begin{align*}
	\sum_{k=2}^d \mathrm{tr}(V_k) &= \frac\pi3\sum_{k=2}^d \mathrm{tr}\big((I-B_{1:(k-1),1:(k-1)})^\top(I-B_{1:(k-1),1:(k-1)})\big) \\
	&=\frac{\pi}{3}\left(\frac{d(d-1)}{2}+\sum_{i=2}^{d-1}\sum_{j=1}^{i-1}(d-i){B_{ij}^*}^2\right)
\end{align*}
The asymptotic variance of the MLE can be computed similarly. 

For engression, we consider $d=2$ and denote the only parameter by $b=B_{21}$. Let $M=(I-B)^{-1}$. Then the loss can be written as
\begin{equation*}
	\ell_{\rm eng}(b;x) = \underbrace{\E\|x-M\varepsilon\|_2}_{m(b;x)} - \frac12\underbrace{\E\|M(\varepsilon-\varepsilon')\|_2}_{c(b)}.
\end{equation*}
Let $r:=X-M\varepsilon$, $s(u):=u_2/\|u\|_2$, and $v=\varepsilon-\varepsilon'\sim\cN(0,2I_2)$. We first compute the gradient:
$\psi_1(b;x):=\frac{\partial}{\partial b}m(b;x)=-\E[s(r)\varepsilon_1]$ and $\psi_2(b)=-\frac12\frac{\partial}{\partial b}c(b)=-\frac12\E[s(Mv)v_1]$. We take $b^*=0$ without loss of generality for computing the variance.  Then $M^*=I_2$ and $r=X-\varepsilon=\varepsilon'-\varepsilon=:w\sim\cN(0,2I_2)$. Note that $$\E\left[\frac{\partial}{\partial b}\ell_{\rm eng}(b^*;X)\right]=-\E\left[\frac{w_2\varepsilon_1}{\|w\|_2}\right]-\frac12\E\left[\frac{v_1v_2}{\|v\|_2}\right]=0.$$Then
\[
\mathrm{Var}\left[\frac{\partial}{\partial b}\ell_{\rm eng}(b^*;X)\right]=\mathrm{Var}(\psi_1(b^*;X))=\E[\psi_1(b^*;X)^2]=\E\left[\E\left(\frac{(X_2-\varepsilon_2)\varepsilon_1}{\|X-\varepsilon\|_2}\mid X\right)^2\right] \approx 3.8\times10^{-3}.
\]
We then compute the hessian. 
\begin{align*}
	\E\left[\frac{\partial^2}{\partial^2 b}\ell_{\rm eng}(b^*;X)\right] &= \E\left[\frac{\partial}{\partial b}\psi_1(b^*;X)\right] + \frac{\partial}{\partial b}\psi_2(b^*) \\
	&= -\E\left[\frac{\partial}{\partial b}S(r)\varepsilon_1\right] - \frac12\E\left[\frac{\partial}{\partial b}S(Mv)v_1\right] \\
	&= -\E\left[-\frac{\varepsilon_1^2}{\|r\|_2}+\frac{r_2^2\varepsilon_2^2}{\|r\|_2^3}\right] - \frac12\E\left[\frac{v_1^2}{\|v\|_2}-\frac{v_2^2v_1^2}{\|v\|_2^3}\right] \\
	&= \frac{7}{32}\sqrt{\pi} - \frac{3}{16}\sqrt{\pi} = \frac{\sqrt{\pi}}{32}.
\end{align*}
Hence the asymptotic variance of $\hat{b}_{\mathrm{eng}} $ is approximately $\frac{32^2}{\pi}\cdot 3.8\times10^{-3}\approx1.2386$.
\end{proof}

\section{Experimental details and additional results}

For all experiments, we use the Adam optimizer with a learning rate of $1\times10^{-4}$, except for the estimation efficiency in Section~\ref{sec:linsem} for which we use a learning rate of $1\times10^{-2}$. For Gaussian mixture experiments,  we use multilayer perceptrons (MLPs), same as what was used in \citet{shen2024engression}, with 5 layers and 512 neurons per layer. For the climate application, we use convolutional neural network for the last step (from $64\times64$ or $32\times32$ or $16\times16$ to $128\times128$) and MLPs for the remaining steps for which we vectorize the spatial fields, as empirically it does not make a difference to keep using convolutional layers for the low-resolution steps. For regional precipitation data, we used randomly subsampled precipitation fields across all RCM models.

%\subsection{Additional results for the synthetic example}
Figures~\ref{fig:mog_inter_t5} and \ref{fig:mog_inter_t10} show the generated samples in the intermediate time steps for the total numbers of steps of 5 and 10, respectively, where we can see how the marginal distributions evolve in the process. With $T=10$, each reverse Markov conditionals are easier to learn, which also shows up in the final generated samples.

%Figure~\ref{fig:rankhist} shows the rank histograms of . 
Figure~\ref{fig:precip_all_t} shows the generated samples for each time step and illustrates how the distributions of the reverse Markov process evolve in this case. 

\begin{figure}[hb]
	\centering
	\includegraphics[width=\textwidth]{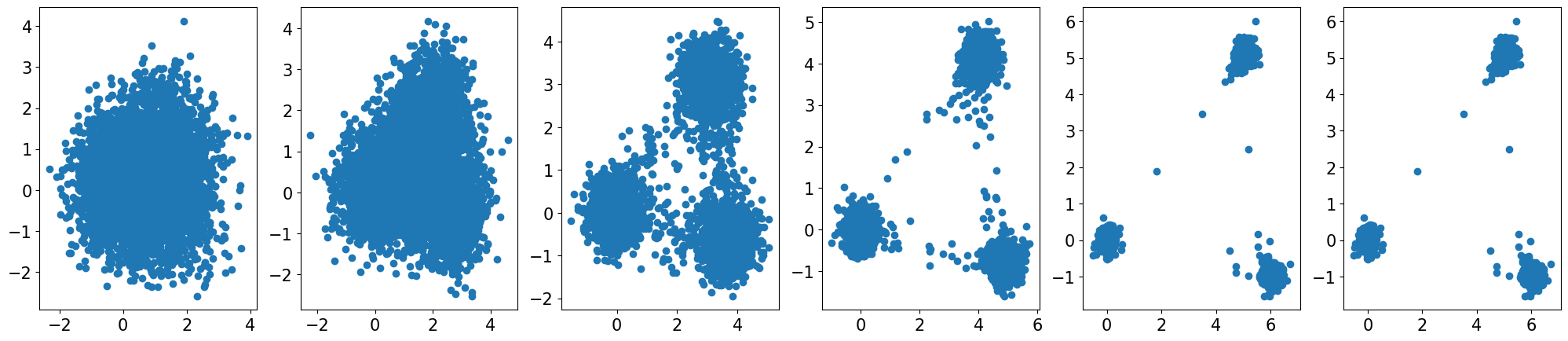}
	\caption{Intermediate steps for $T=5$ for the synthetic example.}\label{fig:mog_inter_t5}
\end{figure}
\begin{figure}[hb]
	\centering
	\includegraphics[width=\textwidth]{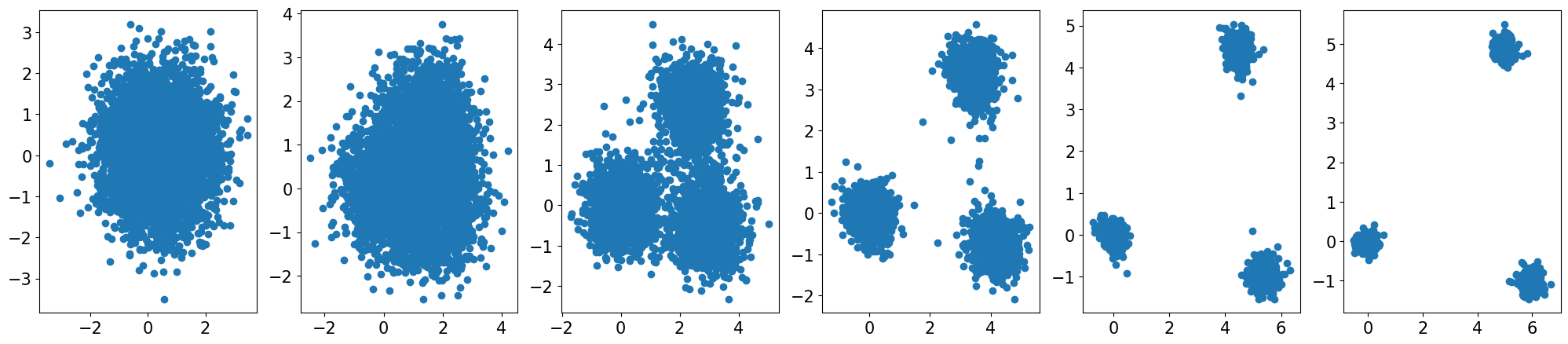}
	\caption{Intermediate steps for $T=10$ for the synthetic example.}\label{fig:mog_inter_t10}
\end{figure}

%\subsection{Additional results for climate prediction}

%\begin{figure}
%	\centering
%	\begin{tabular}{ccc}
%		\includegraphics[width=0.3\textwidth]{fig/precip_metrics/rankhist_t3.pdf} &
%		\includegraphics[width=0.3\textwidth]{fig/precip_metrics/rankhist_t4.pdf} &
%		\includegraphics[width=0.3\textwidth]{fig/precip_metrics/rankhist_t7.pdf}
%	\end{tabular}
%	\caption{Rank histograms }\label{fig:rankhist}
%\end{figure}

\begin{figure}
	\centering
	\begin{tabular}{@{}c@{}ccc@{}}
		&Factor of 2 ($T=7$) & Factor of 4 ($T=4$) & Factor of 8 ($T=3$)\\
		\rotatebox[origin=c]{90}{\small{$t=0$}}\hspace{4pt}\smallskip & \includegraphics[width=.18\textwidth,align=c]{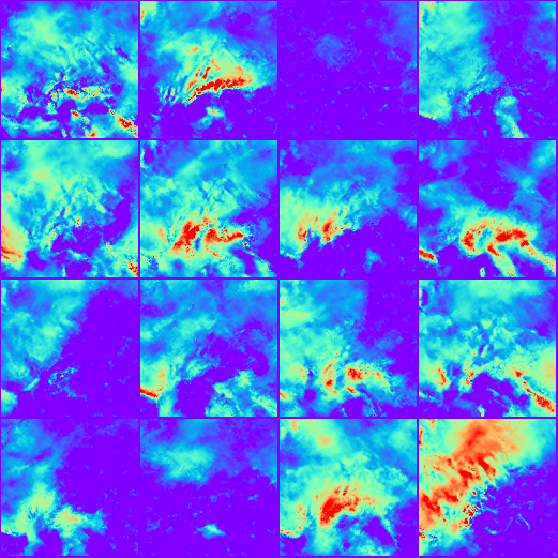}\hspace{4pt} &
		\includegraphics[width=.18\textwidth,align=c]{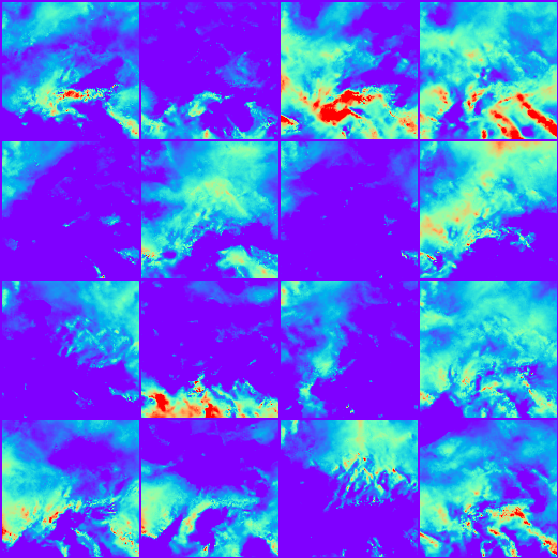}\hspace{4pt} &
		\includegraphics[width=.18\textwidth,align=c]{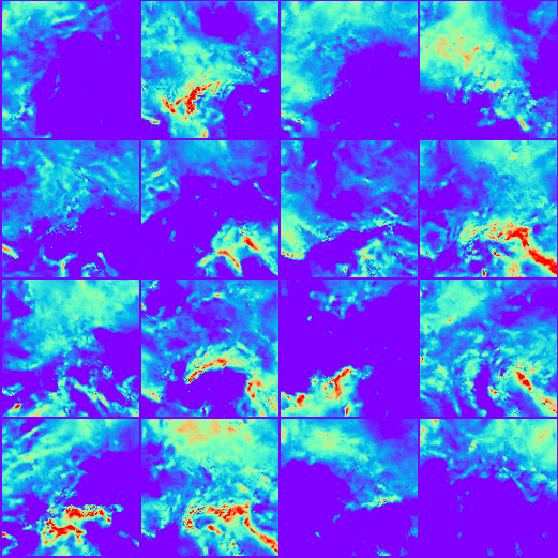} \\
		\rotatebox[origin=c]{90}{\small{$t=1$}}\hspace{4pt}\smallskip & \includegraphics[width=.18\textwidth,align=c]{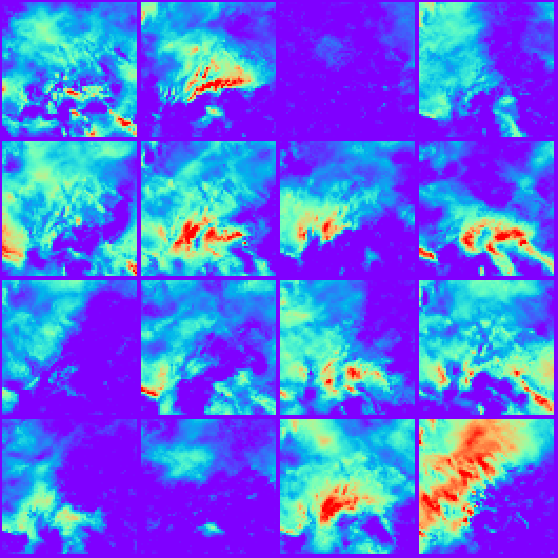}\hspace{4pt} & &\\
		\rotatebox[origin=c]{90}{\small{$t=2$}}\hspace{4pt}\smallskip & \includegraphics[width=.18\textwidth,align=c]{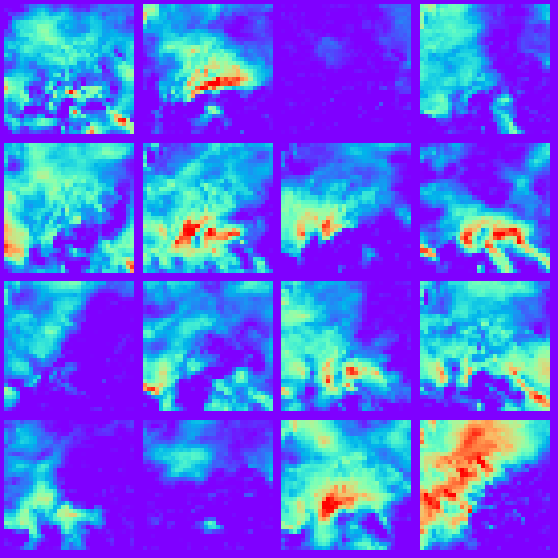}\hspace{4pt} &
		\includegraphics[width=.18\textwidth,align=c]{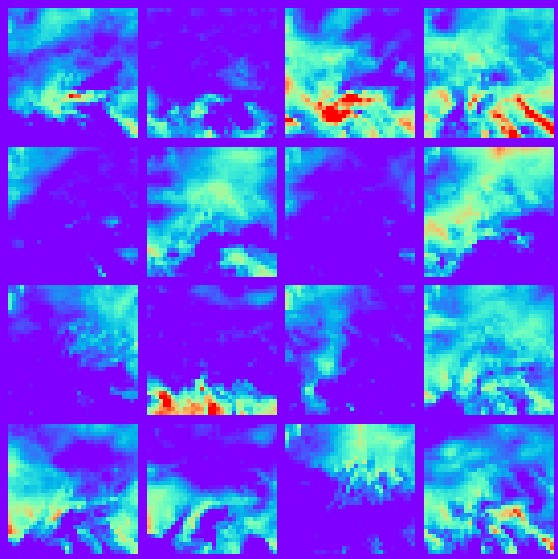}\hspace{4pt} &
		 \\
		\rotatebox[origin=c]{90}{\small{$t=3$}}\hspace{4pt}\smallskip & \includegraphics[width=.18\textwidth,align=c]{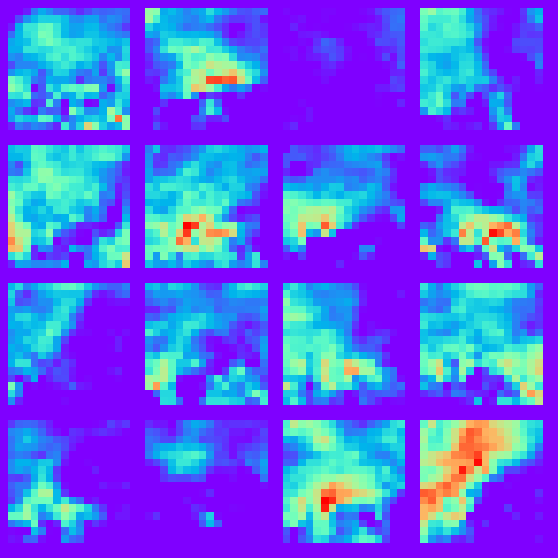}\hspace{4pt} & \hspace{4pt}&
		\includegraphics[width=.18\textwidth,align=c]{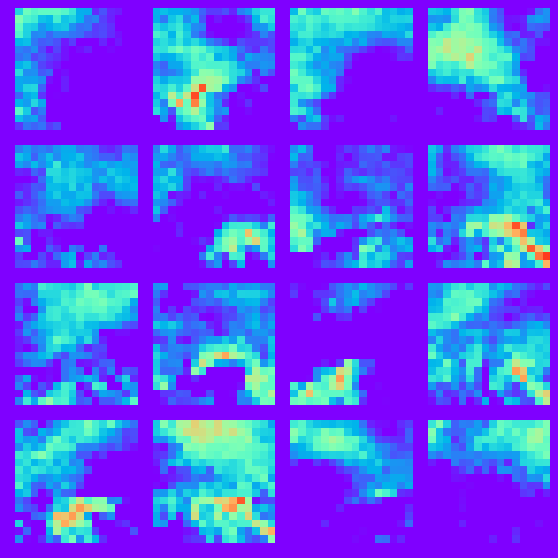}\\
		\rotatebox[origin=c]{90}{\small{$t=4$}}\hspace{4pt}\smallskip & \includegraphics[width=.18\textwidth,align=c]{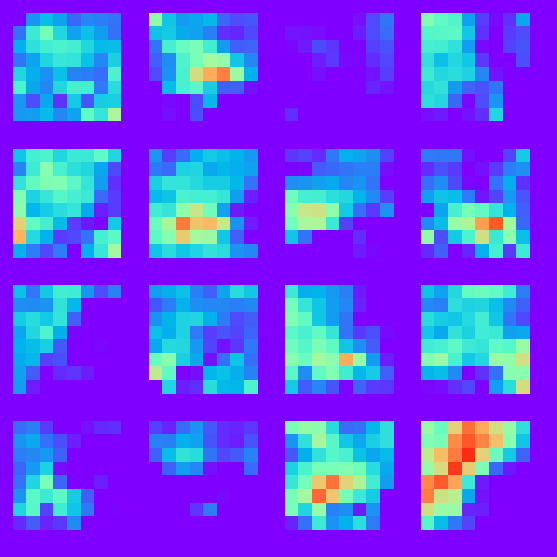}\hspace{4pt} &
		\includegraphics[width=.18\textwidth,align=c]{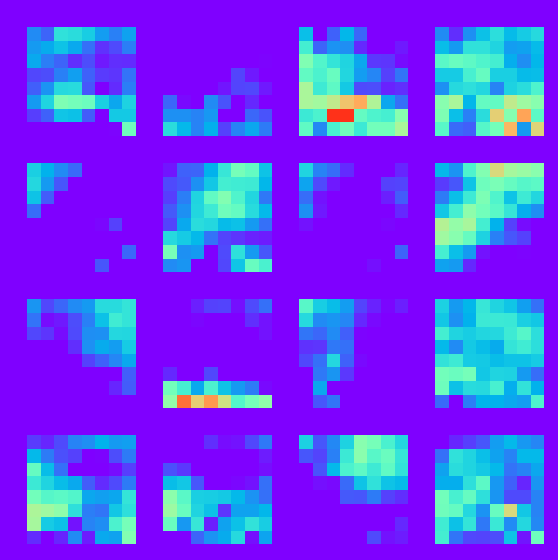}\hspace{4pt} &		\\
		\rotatebox[origin=c]{90}{\small{$t=5$}}\hspace{4pt}\smallskip & \includegraphics[width=.18\textwidth,align=c]{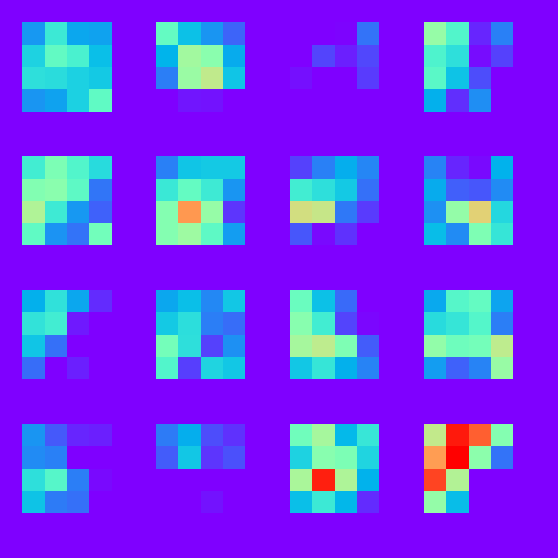}\hspace{4pt} & & \\
		\rotatebox[origin=c]{90}{\small{$t=6$}}\hspace{4pt}\smallskip & \includegraphics[width=.18\textwidth,align=c]{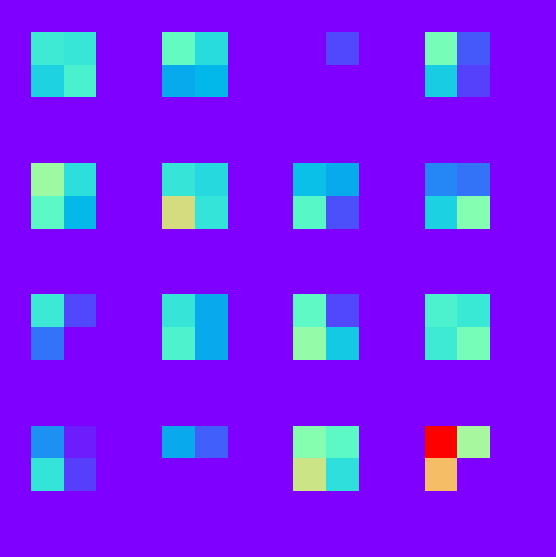}\hspace{4pt} &
		\includegraphics[width=.18\textwidth,align=c]{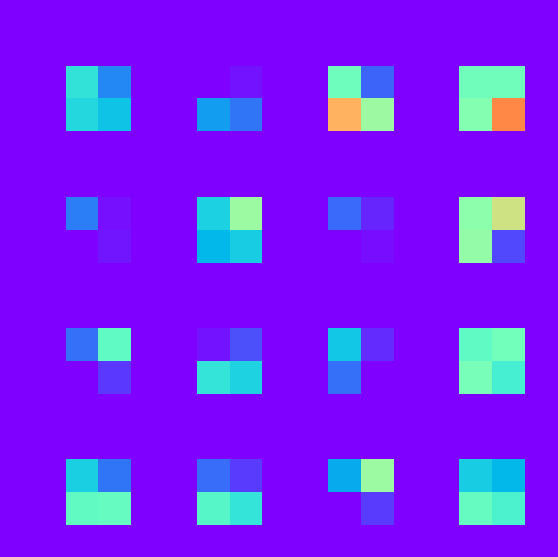}\hspace{4pt} &
		\includegraphics[width=.18\textwidth,align=c]{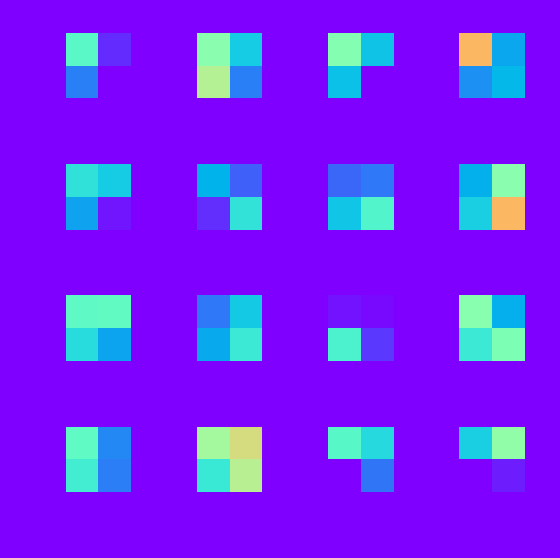} 	\vspace{-0.1in}	
	\end{tabular}
	\caption{Generated samples at different time steps (resolutions) for monthly precipitation data.}\label{fig:precip_all_t}
\end{figure}

%\subsection{Experimental details}

\end{document}